%% file: root.tex
\documentclass[letterpaper, 10pt, journal, twoside]{IEEEtran}
\IEEEoverridecommandlockouts  % Needed for \thanks command
\usepackage[bottom]{footmisc}% places footnotes at page bottom

\usepackage{siunitx}
\usepackage{amsmath}
\usepackage{amsthm}
\usepackage{amssymb}
\usepackage{mathtools}
\usepackage{mathrsfs}

\usepackage{algorithm}
\usepackage[noend]{algpseudocode}

\usepackage{multirow}
\usepackage{enumerate}
\usepackage{tabularx}
\usepackage[dvipsnames]{xcolor}
\usepackage[noadjust]{cite}
\usepackage{indentfirst}
\usepackage{graphicx} 
\graphicspath{{./fig/}}
\usepackage[font=small]{caption}
\usepackage{subcaption}
\usepackage{booktabs}
\usepackage{comment}
\usepackage[breaklinks=true, colorlinks, bookmarks=true]{hyperref}

\usepackage{tikz}
\usepackage{pgfplots}
\pgfplotsset{compat=1.9}

\newtheorem{theorem}{Theorem}[section]

\newtheorem{lemma}[theorem]{Lemma}

\newtheorem{corollary}[theorem]{Corollary}
\newtheorem{proposition}[theorem]{Proposition}
\newtheorem{problem}{Problem}
\newtheorem{assumption}{Assumption}

\newcommand{\oprocendsymbol}{\hbox{$\bullet$}}
\newcommand{\oprocend}{\relax\ifmmode\else\unskip\hfill\fi\oprocendsymbol}

\newcommand{\stkout}[1]{\ifmmode\text{\sout{\ensuremath{#1}}}\else\sout{#1}\fi}

\definecolor{paleGreen}{rgb}{.3, .7, .3}

\newcommand{\interior}[1]{\operatorname{int}(#1)}

\newcommand{\unom}{\bfu_{\mathrm{nom}}}

\newcommand{\Cc}{\mathcal{C}}

\newcommand{\Uc}{\mathcal{U}}

\newcommand{\Ic}{\mathcal{I}}

\renewcommand{\theenumi}{(\roman{enumi})}

\input{cls/sym}

\allowdisplaybreaks

\begin{document}
\title{\LARGE \bf \textsc{BubbleCBF}: Constructive Safety for Second-Order Systems and Application to Autonomous Robot Navigation in Unknown Environments}
\author{%
Mani Amani\textsuperscript{1,*},
Mohammed Alyaseen\textsuperscript{2,*},
Zhirui Dai\textsuperscript{1},
Ki Myung Brian Lee\textsuperscript{1},
Nikolay Atanasov\textsuperscript{1},
and Jorge Cortes\textsuperscript{2}
\thanks{Mani Amani and Mohammed Alyaseen contributed equally.}
\thanks{\textsuperscript{1}Department of Electrical and Computer Engineering;
\textsuperscript{2}Department of Mechanical and Aerospace Engineering,
University of California San Diego, La Jolla, CA, USA.}
\thanks{We gratefully acknowledge support from ONR Award N00014-23-1-2353 and ARL DCIST CRA W911NF-17-2-0181.}
}
\maketitle
\thispagestyle{empty}  % remove page number
    \pagestyle{empty}      % remove page number
\input{tex2/abstractv2}

\input{tex2/introv2}

\input{tex2/rel_work}

\input{tex2/prob_formv2}

\input{tex2/solveProblemForm}
\input{tex2/planning}
\input{tex2/results}

\input{tex2/conclusion}

\bibliographystyle{IEEEtran}
\bibliography{refs}
% \bibliographystyle{unsrt}
% \newpage

\input{appendix}
%\appendices
\end{document}

%% file: cls/sym.tex
\newcommand{\bfc}{\mathbf{c}}

\newcommand{\bfe}{\mathbf{e}}
\newcommand{\bff}{\mathbf{f}}

\newcommand{\bfq}{\mathbf{q}}

\newcommand{\bfu}{\mathbf{u}}
\newcommand{\bfv}{\mathbf{v}}

\newcommand{\bfx}{\mathbf{x}}
\newcommand{\bfy}{\mathbf{y}}
\newcommand{\bfz}{\mathbf{z}}

\newcommand{\calB}{\mathcal{B}}
\newcommand{\calC}{\mathcal{C}}

\newcommand{\calH}{\mathcal{H}}

\newcommand{\calO}{\mathcal{O}}

\newcommand{\calQ}{\mathcal{Q}}

\newcommand{\calS}{\mathcal{S}}

\newcommand{\calV}{\mathcal{V}}

\newcommand{\calX}{\mathcal{X}}

%% file: tex2/abstractv2.tex
\begin{abstract}
Satisfying safety constraints for systems with second-order dynamics is challenging because a position constraint may require infeasible deceleration. Higher-order control barrier functions (CBFs) provide an approach for expressing safety conditions for such systems but feasibility depends on the system dynamics and available control authority. We develop a constructive approach for designing feasible safety constraints for second-order control-affine systems. Starting from a signed distance function (SDF), we construct collision-free spheres, i.e., \emph{bubbles}, in position space, and use their quadratic geometry and overlap to construct a velocity-dependent safety constraint. The associated safe set is represented as a union of intersections and enforced through a quadratic-program safety filter. Under mild assumptions on the dynamics, bubble overlap, and design parameters, we prove that the quadratic-program-based controller is well defined, continuous in the system state, and renders the set forward invariant. We also derive sufficient conditions for feasibility under bounded inputs. We demonstrate that our \textsc{BubbleCBF}, applied to a differential-drive robot, achieves safe autonomous navigation in unknown environments. Videos and results are available at:
\url{https://existentialrobotics.org/bubblecbf/}. 
\end{abstract}

%% file: tex2/introv2.tex
\section{Introduction}
Autonomous navigation in unknown environments requires a robot to simultaneously build a map of its surroundings and execute actions that avoid collisions. 
% For robots with second-order dynamics, a central challenge is not only to represent collision-free space but also to ensure the existence of feasible actions that enforce safety constraints. 
Beyond representing obstacle-free space, robots with second-order dynamics must ensure that feasible actions exist to enforce safety constraints. Since actuation enters at the acceleration level, a robot may remain positionally safe while approaching the boundary too quickly for any admissible input to prevent escape.

Control barrier functions (CBFs) \cite{safetycriticalAmes} provide a principled approach for safe control design. A CBF limits how quickly the safe set boundary may be approached with rate dependent on the barrier value.
%Rather than requiring a robot to remain inside a safe set, a CBF limits how quickly the safe set boundary may be approached. This allowable rate depends on the barrier value. 
Far from the boundary, the barrier value may decrease more quickly, while near the boundary, the permitted decrease becomes increasingly restrictive. Signed distance functions (SDFs) are natural CBF candidates because they encode free space using the distance to the nearest boundary \cite{osher2004level}. SDF models of the environment can be estimated from point-cloud observations using data-driven and probabilistic approaches, including neural implicit methods \cite{isdf} and Gaussian-process methods \cite{kernelSDF}. However, for second-order systems, a position-dependent safety constraint has relative degree two and cannot be enforced by formulating CBF constraints directly from SDF values. Higher-order CBFs (HOCBFs) \cite{highorderCBF} formulate safety constraints for high-relative-degree systems but do not by themselves guarantee the existence of feasible control inputs. A key challenge in safe control synthesis for second-order systems is therefore to construct valid safe sets in the joint position-velocity space whose invariance can be enforced with admissible inputs.

% are a natural representation of collision geometry because they 
%These techniques enable SDF estimation even when the environment is only partially observed or measurements are uncertain.
% \NA{What is this prescribed rate? Explain that it is related to the distance from the boundary of the safe set.}

%safe set on which the safety-filtering problem is guaranteed to remain feasible.

\begin{figure}
\centering
\includegraphics[width=\linewidth,trim={0 1cm 0 0.85cm},clip]{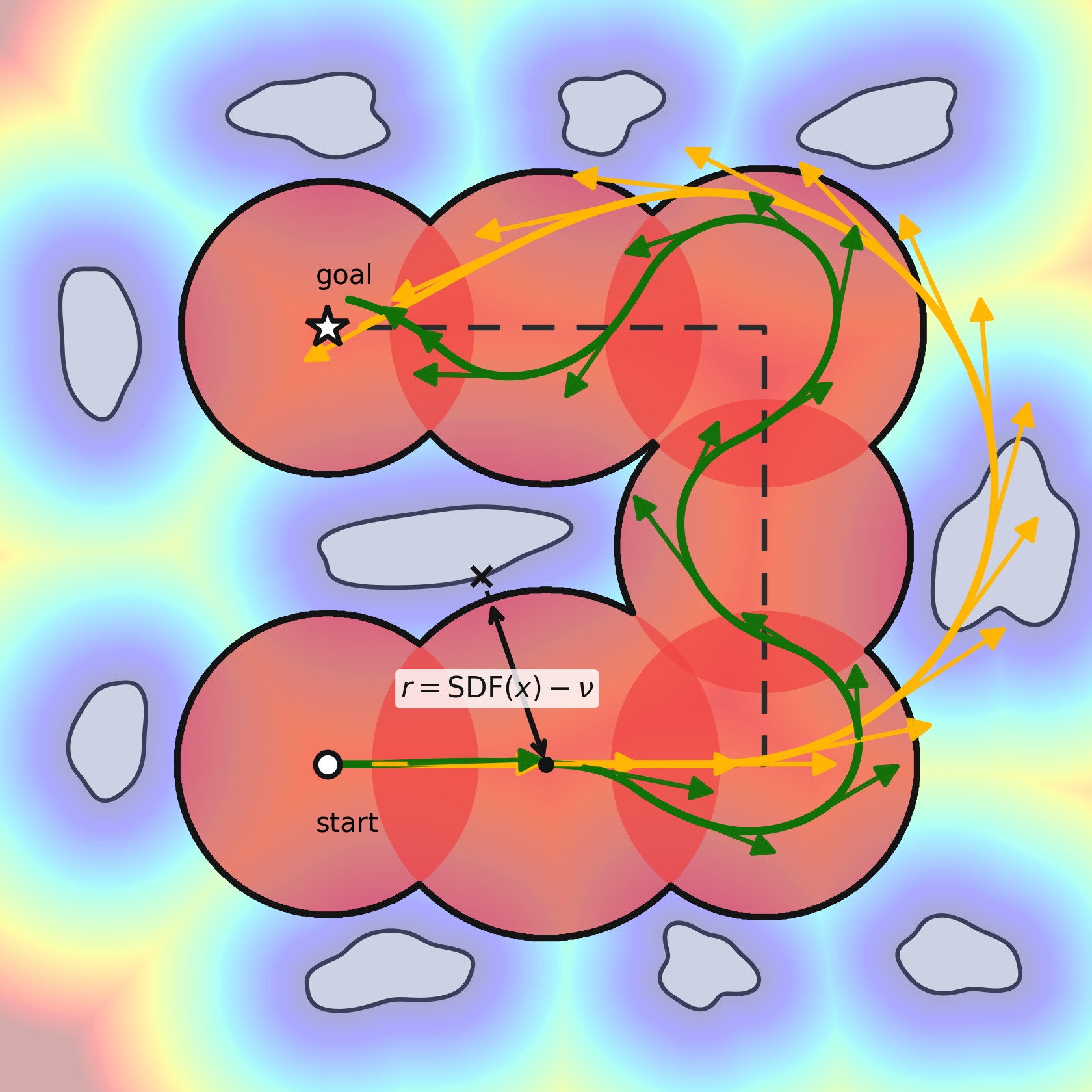}
\caption{\textsc{BubbleCBF} guides a second-order system from start to goal through a safe corridor consisting of overlapping bubbles. Each bubble is a ball with radius underestimating the distance to the nearest obstacle. The union of bubbles is lifted to a velocity-dependent safe set to account for the relative-degree-two dynamics. A quadratic program with a control barrier function constraint defined from the lifted safe set is used to modify a nominal controller to restrict the system's outward motion near bubble boundaries and ensure sufficient control authority for safety. The nominal system trajectory (orange) and the \textsc{BubbleCBF}-filtered trajectory (green) are shown.}
\label{fig:mainfig}
\end{figure}

%  in position space

%\NA{In the caption of Fig.~\ref{fig:mainfig}, we say "obstacles induce a collection of overlapping safe bubbles". How? Why? This is too mysterious. The goal of the figure is to highlight our contribution. Currently this is not clear. What is new in this paper? Why are second-order systems important?},
%geometric bubbles. Each bubble is a ball inscribed in free space and can be obtained from an SDF query.
%\cite{lee2024safe}
% \NA{I suggest not including a citation right in the sentence where we introduce our method. This can bring questions about novelty.}
% \NA{We should talk about SDF and SDF estimation methods more earlier so that the reference here is clear. Currently, we have only 1 sentence. We can add 1-2 more to explain that it is a widely used representation in robotics and that neural network and GP methods have been developed for its estimation from point-cloud date.}.
% \NA{What is this prescribed rate? Explain that it is related to the distance from the boundary of the safe set.}

To address this challenge, we develop \textsc{BubbleCBF}, a safe control method for second-order control-affine systems. Starting from an SDF representation of the obstacles, we derive a velocity-dependent safe set from overlapping collision-free spheres, i.e., \emph{bubbles}, illustrated in Fig.~\ref{fig:mainfig}. Our design extends the union-of-intersections CBF approach of \cite{MA-NA-JC:24-auto} to quadratic constraints and guarantees forward-invariance of the safe set and feasibility under input constraints. Our contributions are summarized as follows.

\begin{itemize}
    \item We develop \textsc{BubbleCBF}, a safe control method for second-order control-affine systems that constructs velocity-dependent CBFs from position-space bubbles.
    
    %present \textsc{BubbleCBF}, an optimization-based safety-filtering approach for second-order control-affine systems. Using only SDF values, \textsc{BubbleCBF} constructs a union of collision-free bubbles and exploits their quadratic geometry to define a velocity-dependent safe set without requiring SDF gradients or Hessians.

    \item Under mild assumptions on the system dynamics, bubble overlap, and input constraints, we prove that the quadratic program (QP) for \textsc{BubbleCBF} is feasible over the safe set, and renders the set forward invariant. 
    %We derive sufficient conditions for feasibility under input constraints.

    \item \textsc{BubbleCBF} enables safe autonomous navigation with a Clearpath Jackal robot in real-world environments, achieving a 100\% success rate, compared to 50\% for an SDF-CBF baseline, while increasing average obstacle clearance by 4.8\%. 
\end{itemize}

%% file: tex2/rel_work.tex
\section{Related Work} 
CBFs have become a popular tool for safety-critical robot control~\cite{safetycriticalAmes,KehanDR,CBFdeadlock}. For collision avoidance, SDFs~\cite{osher2004level} are a natural candidate to form safety constraints because they measure the minimum distance to occupied space. However, using an SDF to define a CBF constraint directly leads to challenges because SDF are often nonsmooth and can result in high relative degree constraints for high-order systems.

% Existing work addresses these challenges in several ways. In~\cite{thirugnanam2022safety}, smooth safety constraints are derived from nonsmooth SDFs via duality. Other approaches approximate SDFs with smooth functions~\cite{wu2025optimization,long2021learning}, while \cite{glotfelter2017nonsmooth} builds on nonsmooth CBF theory to handle nonsmooth constraints directly. High relative degree is typically addressed by constructing higher-order CBFs \cite{highorderCBF,nguyen2016exponential,MA-NA-JC:25-csl}.
% % \NA{We should provide more references here. There is a lot of work on high-order CBFs}. 
% %Methods that handle nonsmoothness directly generally consider simple first-order models, avoiding the relative-degree issue.
% Barbosa and Tumova \cite{barbosa2022risk} address both challenges associated with SDF safety constraints by first approximating the SDF with a smooth function and then applying higher-order CBF techniques.
Existing work addresses these challenges in several ways. In~\cite{thirugnanam2022safety}, dual variables are used to formulate CBF constraints for nonsmooth distance functions between polytopes. Other approaches construct smooth approximations of distance functions~\cite{wu2025optimization}, while nonsmooth CBF theory handles nonsmooth barrier compositions directly~\cite{glotfelter2017nonsmooth}. High relative degree is addressed using exponential and higher-order CBF constructions~\cite{nguyen2016exponential,highorderCBF}. Long et al.~\cite{long2021learning} learn SDF approximations online and use them to construct robust CBF constraints for safe navigation in unknown environments. Barbosa and Tumova~\cite{barbosa2022risk} address both nonsmoothness and high relative degree by constructing a smooth approximation of the distance field and its derivatives and applying higher-order CBF techniques.
% \NA{Is this the only work that does safe robot navigation with SDF constraints?}
In contrast, we neither smooth nor differentiate the SDF. We use only SDF queries to construct collision-free balls \cite{lee2024safe}, whose quadratic geometry allows us to construct a velocity-dependent safe subset of their union for second-order systems.
% \NA{Rewrite the second part of this sentence. How is our work different from \cite{barbosa2022risk}?}.

Regularity is especially challenging for nonsmooth safe sets composed of multiple barrier functions~\cite{PM-AA-JC:25-ejc,MA-NA-JC:25-tac,morris2015continuity}. Alyaseen et al.~\cite{MA-NA-JC:24-auto} address this by enforcing componentwise CBF constraints with state-dependent relaxation, avoiding explicit switching between active components and yielding a continuous optimization-based controller. We extend this approach to a union of bubble intersections and derive the corresponding quadratic barriers and feasibility conditions for second-order dynamics and bounded inputs. 
% After designing a valid CBF comes the nontrivial task of using the designed CBF to construct a safe controller with suitable regularity properties\NA{This sentence is not saying anything useful.}. Optimization-based feedback controllers, especially quadratic programming-based controllers~\cite{SCH-XX-ADA:15} have gained popularity in this regard due to their ease of use, flexibility to encode other control objectives, and relatively low computational demand. Those controllers are solutions of optimization problems parameterized by the state and constrained by the safety requirements\NA{Isn't this what CBF methods that we reviewed in the previous paragraph essentially are? This paragraph makes it sound like we are talking about something else.}. Regularity properties, such as continuity, Lipschitzness, and boundedness, of this class of controllers are studied in \cite{PM-AA-JC:25-ejc,MA-NA-JC:25-tac,morris2015continuity}. Ensuring regularity properties is especially challenging when using nonsmooth or multiple barrier functions because of the abrupt changes in safety requirements with the state. This difficulty is circumvented in \cite{MA-NA-JC:24-auto}, which constructs a continuous controller for a safe set given as a union of different set components\NA{This is vague. What are these different set components? How is the difficulty circumvented?}. We deploy the design in~\cite{MA-NA-JC:24-auto} in our work\NA{How do we employ it? What is different? What is new?}. %\BL{Mo: please add some control design literature.}

Bubble-based representations have also been used in robotics to cover collision-free configuration space with overlapping local regions~\cite{BubblePlanner,lee2024safe}. Sampling-based neighborhood graphs (SNG)~\cite{yang2004sng} cover free configuration space with overlapping neighborhoods, such as balls, and construct feedback navigation functions over the resulting graph. Invariant configuration-space bubbles~\cite{danielson2023invariant} lift configuration-space bubbles to control-invariant subsets of the state space and enforce invariance using a command governor. In contrast, we construct a CBF over a union of overlapping bubbles and enforce it using a continuous QP safety filter for general second-order systems.
Moreover, we show that this QP is always feasible within the safe set, and hence guarantees forward invariance of the safe set. Geometry-aware mixed-integer linear programming model predictive control (MILP-MPC)~\cite{chen2026long} combines long-horizon planning with Minkowski-based CBFs, accounting explicitly for robot geometry and environment topology at the cost of mixed-integer optimization. Our approach instead operates on an SDF-derived bubble cover, enabling online safety filtering in unknown environments without mixed-integer planning or exact robot-obstacle distance computations.

%% file: tex2/prob_formv2.tex
\section{Problem Formulation}
\label{sec:problem_formulation}

Consider a second-order control-affine system:
\begin{equation}
\label{eq:second_order_dynamics}
\begin{aligned}
    \dot{\bfq} &= \bfv,\\
    \dot{\bfv} &= \bff(\bfq, \bfv)+G(\bfq,\bfv)\bfu,
\end{aligned}
\end{equation}
where $\bfq\in\mathcal{Q}\subseteq\mathbb{R}^n$ is the generalized position, $\bfv\in\mathcal{V}\subseteq\mathbb{R}^n$ is the generalized velocity, and $\bfu\in\mathcal {U}\subseteq\mathbb{R}^m$ is the control input. We denote the system state by $\mathbf x \coloneqq \begin{bmatrix}\bfq^\top & \bfv^\top\end{bmatrix}^{\top} \in \mathcal{X} \coloneqq \mathcal{Q}\times\mathcal{V}$. The functions $\mathbf{f}:\mathcal{X}\to\mathbb{R}^n$ and $G:\mathcal{X}\to\mathbb{R}^{n\times m}$, defining the system dynamics, are assumed to be locally Lipschitz. The system has a maximum speed constraint $\|\mathbf{v}\|\leq v_{\max}$. The second-order dynamics in~\eqref{eq:second_order_dynamics} can be written in control-affine state-space form as
\begin{equation}
\label{eq:bars}
\dot{\bfx}
=
\bar{\bff}(\bfx)+\bar G(\bfx)\bfu,
\quad
\bar{\bff}(\bfx)
\coloneq
\begin{bmatrix}
\bfv\\[-1pt]
\bff(\bfx)
\end{bmatrix},
\quad
\bar G(\bfx)
\coloneq
\begin{bmatrix}
0\\[-1pt]
G(\bfx)
\end{bmatrix}.
\end{equation}
Let $i \in \mathcal{I}\coloneqq\{1,\ldots,N\}$ index a finite collection of
closed balls $\mathcal{B}_i$, specifying safety constraints for the system's position:
\begin{equation}
\label{eq:positional_safe_ball}
    \mathcal{B}_i
    \coloneqq
    \left\{
        \mathbf q\in\mathcal Q
        \;\middle|\;
        b_i(\mathbf q)\geq 0
    \right\},
    \quad
    b_i(\mathbf q)
    \coloneqq
    r_i^2-\lVert\mathbf q-\mathbf c_i\rVert^2,
\end{equation}
where $\mathbf c_i\in\mathcal Q$ and $r_i>0$ are the center and radius of
the $i$-th ball. We refer to each $\mathcal{B}_i$ as a \emph{bubble}.

In practice, bubbles can be obtained by querying an SDF representation of the free space in $\calQ$. Fig.~\ref{fig:mainfig} illustrates this construction. The background depicts the SDF, which is queried at the bubble center $\bfc_i$ to obtain the bubble radius $r_i$. 
%The resulting SDF value provides the radius of collision-free bubble.
Thus, the safety constraint is:
\begin{equation} \label{eq:positional_safety_requirement}
    \mathbf{q}(t)\in\mathcal{B} \coloneqq \bigcup_{i\in\mathcal I}\mathcal B_i
    \qquad \text{for all }t\geq 0.
\end{equation}
% \NA{Give some motivation about why we define the safety constraints this way by connecting to SDF in robotics before discussing the relative degree part.}
% Because each $B_i$ depends only on $\mathbf{q}$, it has relative degree two
% with respect to \eqref{eq:second_order_dynamics}. 
% \NA{Give a reminder and a reference what relative degree is and means.}. 
% Each safety constraint $b_i(\bfq) \geq 0$ depends only on $\bfq$ and, hence, has relative degree two with respect to \eqref{eq:second_order_dynamics}.
%each scalar constraint \(B_i:\mathcal{Q}\to\mathbb{R}\) depends only on \(\mathbf{q}\), it has relative degree two with respect to \eqref{eq:second_order_dynamics}. 
This requirement is on the system position only, whereas control applies to the acceleration. Therefore, an outward velocity may push the system position outside $\mathcal{B}$ before it can be reversed.
% The control input $\bfu$ first appears in $\ddot b_i$, so each bubble constraint has relative degree two \cite{safetycriticalAmes}. 
Thus, it is generally not possible to certify safety for all of $\mathcal{B}\times\mathcal{V}$. We therefore seek a velocity-dependent safe set $\mathcal{S}\subseteq\mathcal{B}\times\mathcal{V}$.

% \MA{clarifying that the set needs to be nonempty/nontrivial}
%Let $\unom:\mathcal{X}\to\mathcal{U}$ be a nominal controller that is not assumed to satisfy \eqref{eq:positional_safety_requirement}.
%\begin{definition}[Safety filtering]
% \label{def:safety_filter}
% Let $\unom:\mathcal{X}\to\mathcal{U}$ be a continuous nominal controller, designed for performance (e.g., goal tracking) and not assumed to satisfy \eqref{eq:positional_safety_requirement}. A \emph{safety filter} is a feedback controller that, at each state, modifies $\unom$ as little as possible while guaranteeing that the closed-loop trajectory satisfies the safety requirement; in this work it takes the form of a state-dependent optimization whose solution is $\bfu^*(\bfx)$.
% \end{definition}

%
% \marginJC{Shouldn't we drop the idea for the reader that we want $\mathcal{S}$ to be large or nontrivial? E.g., $\mathcal{S} = \mathcal{B} \times \{0\}$ will work for our problem statement below, but we hardly would like that.}
%

\begin{problem}[Safe control over finite union of bubbles]
\label{prob:bubble_safety_filtering}
Given the system in \eqref{eq:second_order_dynamics}, and a continuous nominal controller $\unom:\mathcal{X}\to\mathcal{U}$,
%for the nominal task (e.g., goal tracking), and not assumed to satisfy \eqref{eq:positional_safety_requirement}, 
construct a nonempty set $\mathcal{S}\subseteq\mathcal{B}\times\mathcal{V}$ and a feedback controller $\bfu^*:\mathcal {S}\to\mathcal{U}$, such that:
\begin{enumerate}
    \item[(i)] $\mathcal{S}$ is forward invariant for the closed-loop system under $\bfu=\bfu^*(\bfx)$, i.e., every trajectory initialized in $\calS$ stays in $\calS$ for all $t\geq 0$; 
    %\NA{This seems to be implied by (i). Is this needed?}
    
    %\item[(ii)] the safety-filtering optimization defining $\mathbf u^*(\mathbf x)$ is feasible for every $\mathbf x\in\mathcal S$; \NA{Rewrite this. We have not introduced what "safety-filtering optimization" is and it is better not to. Either remove this entirely or you can write something like "The feedback controller $\mathbf{u}^*$ is well-defined for every $\mathbf x\in\mathcal{S}$."} and
    
    \item[(ii)] $\bfu^*$ minimally deviates from the nominal controller $\unom$.
    %among all safe admissible inputs\NA{The thing after is not an input but a controller.}, $\mathbf{u}^*(\mathbf x)$ minimally modifies $\unom(\mathbf x)$.
\end{enumerate}
\end{problem}

We refer to such a controller as a \emph{safety filter}.
 % \NA{Where is Fig.~\ref{fig:overall_piepline} referenced? Change its caption to avoid meaningless words like "overall", "system", "pipeline" and make it more descriptive about what is going on.}

%% file: tex2/solveProblemForm.tex
% develop \textsc{BubbleCBF} to 

\section{BubbleCBF Safety Filter}
We solve Problem~\ref{prob:bubble_safety_filtering} by constructing a velocity-dependent safe set $\calS$ as a union of sets $\calS_i \subset \calX$, each corresponding to a bubble $\calB_i$. We design the controller $\bfu^*$ as the solution of a QP subject to a CBF constraint encoding $\calS$. We impose additional constraints on the input $\bfu$ to ensure that $\calS$ is forward invariant and the QP is feasible.

%e then construct affine constraints on $\bfu$ that render $\calS$ forward invariant and design the controller $\bfu^*$

%so that the safety filter in Problem~\ref{prob:bubble_safety_filtering} can be realized by solving a QP.
%The QP is shown to be feasible even with input constraints, which formally guarantees the forward invariance of $\calS$.
%\NA{What is the difference between feasibility and compatibility with input bounds? Feasibility is probably enough.}. 

\subsection{Control Design}

To enforce the positional safety requirement in \eqref{eq:positional_safety_requirement}, we first consider the forward invariance of a single bubble.
%and extend this construction to the union of multiple bubbles.
A continuously differentiable function \(B: \mathcal{X} \rightarrow \mathbb{R}\) is called a \emph{CBF} \cite{safetycriticalAmes} for the system~\eqref{eq:bars} if there exists an extended class-$\mathcal{K}$ function $\Gamma$ and an admissible control input at every state satisfying $\dot{B} + \Gamma(B) \geq 0$ along the trajectory. Such control inputs render the super-level set $\{\bfx \in \mathcal{X}  \mid B(\bfx) \geq 0 \}$ forward invariant.
% at every state in that domain, an admissible input satisfies
% \begin{equation}
% \label{eq:cbf_general}
% \mathcal L_{\bar{\bff}}B(\bfx)
% +\mathcal L_{\bar G}B(\bfx)^\top\bfu
% +\gamma(B(\bfx))\geq0.
% \end{equation}
Choosing $\Gamma(B) = \gamma B$ for a scalar $\gamma > 0$, the CBF condition for the bubbles in~\eqref{eq:positional_safe_ball} is:
\begin{equation}
\label{eq:firstorderCBF}
\dot b_i+\gamma b_i=-2(\bfq-\bfc_i)^\top\bfv+\gamma b_i(\bfq)
\geq0.
\end{equation}
Here, the input $\bfu$ does not appear in
\eqref{eq:firstorderCBF}, since it acts through acceleration.
To obtain a condition depending on $\bfu$, a HOCBF~\cite{highorderCBF} treats the LHS of~\eqref{eq:firstorderCBF} as another CBF:
\begin{equation}
    h_i(\bfq,\bfv)
    \coloneqq \dot b_i+\gamma b_i-\epsilon_i,
\end{equation}
where $\epsilon_i>0$ provides a strict safety margin.
Imposing a CBF condition on $h_i$ as $\dot h_i+\alpha h_i\geq0$ for $\alpha > 0$ leads to $h_i \geq 0$ and, in turn, $b_i \geq 0$ over the trajectories.
As the dynamics~\eqref{eq:second_order_dynamics} is affine in $\bfu$, the inequality $\dot h_i+\alpha h_i\geq0$ is also affine in $\bfu$.
% Differentiating $h_i$ introduces $\ddot b_i$ and hence the
% acceleration input, yielding an inequality affine in $\bfu$.
The resulting safe set in the state space is:
\begin{equation}
\label{eq:cut_set}
    \calH_i
    =
    \left\{
    (\bfq,\bfv)\in\calX
    \mid h_i(\bfq,\bfv)\geq0
    \right\}.
\end{equation}
To account for the maximum speed constraint of the system, we define another CBF candidate:
\begin{equation}\label{eq:velconstraint_set}
    \calC = \{\bfv \in \calV \mid c(\bfv) \coloneq v_{\max}^2 - \|\bfv\|^2 \geq 0\}.
\end{equation}

Now that we have designed an HOCBF for a single bubble, we aim to render the union
$\calB:=\bigcup_{i\in\mathcal I}\calB_i$ forward invariant.
% \BL{Perhaps move the $v_{\max}$ to prob form}:
% \NA{Why not use something more meaningful than $C_{\max}$ like $v_{\max}$ below?}
% \begin{equation}\label{eq:velconstraint_set}
%     \calC = \{\bfv \in \calV \mid c(\bfv) \coloneq v_{\max}^2 - \|\bfv\|^2 \geq 0\}.
% \end{equation}
% Given these sets, our safe set candidate over the state space $\calX$ can be defined as a union of intersections of the first- and second-order\NA{I am not sure what a "second-order set" is. It may be best to avoid this.} sets above:
%Given the sets induced by the CBF constraints above, 
We define the safe set over \(\calX\) as a union of intersections:

\begin{equation}\label{eq:safe_set}
% \begin{aligned}
    \calS \coloneq \bigcup_i \calS_i, \quad 
    \calS_i \coloneq (\calB_i \times \calC) \cap \calH_i.
% \end{aligned}
\end{equation}
These intersections $\calS_i$ and their union $\calS$ can be represented implicitly as super-level sets by applying $\min$ and $\max$ operations as follows:
% $\calS$ and the intersection 
% These union and intersection sets can be encoded implicitly as the super-level sets of corresponding CBF candidates using $\min$ and $\max$ functions in place 
% %In terms of CBF candidates,\NA{Explain what this means, i.e., that these sets can be encoded implicitly as the zero-level sets of these functions, which can then be used as CBF constraints in the control design stage.}, 
% the intersection and union operations can be encoded using $\min$ and $\max$ functions as follows:
% Operationally, safe set candidate is represented as a super-level set of the following function $S$.
\begin{equation}
\begin{aligned}
    s_i(\bfx) &= \min\left\{ b_i(\bfq), c(\bfv), h_i(\bfx)\right\}, \\
    s(\bfx) &   = \max_{i} s_i(\bfx),
\end{aligned}
\end{equation}
so that $\calS_i$ is the zero superlevel set of $s_i(\bfx)$ and $\calS$ is the zero superlevel set of $s(\bfx)$.
% \NA{The notation of very similar letters for sets and functions makes it hard to read in my opinion. Perhaps, we can use $s_i$ instead of $S_i$?}, and $\calS$ is the super-level set of $s$. 
% The components $B_i(\bfq)$, $C(\bfv)$, and $H_i(\bfx)$ are defined in~\eqref{eq:positional_safe_ball},~\eqref{eq:cut_set}, and~\eqref{eq:velconstraint_set}\NA{This sentence is not needed.}.
However, these functions are not suitable for use as CBFs, because they are nonsmooth: the minimal constraint in \(s_i\) and the maximal index in \(s\) may change with the state. 
Such non-smoothness violates the theoretical premise of CBF constraints~\cite[Theorem 2]{safetycriticalAmes}, and can cause the system to leave $\calS$~\cite[Example 4.4]{MA-NA-JC:24-auto}.
Instead of differentiating \(s\), we instead impose CBF conditions on \(b_i\), \(c\), and \(h_i\) for every index \(i\), while relaxing those associated with non-maximal indices.
%\NA{Should we give this function a name and refer to it after?}

To see which bubble index $i$ contributes to the maximum, define the nonnegative gap:
\begin{equation}
    \Delta_i(\bfx)= s(\bfx)-s_i(\bfx).
\end{equation}
%
% \marginJC{All this explanation about constructing a relaxation for the inactive constraints, we authors know why, but we have not explained the reader why we're doing it. I'd say something along the lines that we don't want to only consider the active constraints, b/c this set changes with the state, and that can induce discontinuities in the controller as a function of the state. Our solution is to consider all constraints (that's the all-components bit from the abstract), but doing it in a way that does not overconstrain the control action, i.e., making sure that satisfying the non-active constraints is easy at any given state.}
%
We enforce safety using the following QP, where the CBF constraints are imposed for every bubble \(i\):
\begin{equation}
\label{eq:qp_cbf}
\begin{aligned}
\min_{\alpha,M,\bfu}\quad
& \|\bfu-\unom(\bfx)\|_{P}^{2}
  + q_{\alpha}\alpha^{2}
  + q_M M^{2} \\
\text{s.t.}\quad
& M \ge \mu_M,
  \qquad
  \alpha \ge \mu_\alpha, \\
\forall i, \quad 
& \mathcal \nabla b_i(\bfq)^\top \bfv
  + \gamma b_i(\bfq)
  + M\Delta_i(\bfx)
  \ge 0, \\
\forall i, \quad 
& \mathcal L_{\bar{\bff}} h_i(\bfx)
  + \mathcal L_{\bar G} h_i(\bfx)^\top \bfu
  + \alpha h_i(\bfx)
  + M\Delta_i(\bfx)
  \ge 0, \\
\forall i, \quad 
& \mathcal L_{\bff} c(\bfv)
  + \mathcal L_{G} c(\bfv)^\top \bfu
  + \alpha c(\bfv)
  + M\Delta_i(\bfx)
  \ge 0.
\end{aligned}
\end{equation}
Here, $\|z\|_P^2\coloneqq\bfz^\top P\bfz$, with $P\succ0$, penalizes deviation from the nominal input, and $M$ controls inactive-component relaxation. 
The coefficients $q_\alpha>0$ and $q_M>0$ regularize $\alpha$ and $M$. Allowing $\alpha$ and $M$ to vary ensures feasibility at every state~\cite{MA-NA-JC:24-auto}. At a fixed state $\bfx$, the constraints are affine in $(\bfu,\alpha,M)$, so the problem~\eqref{eq:qp_cbf} is a QP.

%{\MMA{1. why only $\forall i$ in the 1st line of constraints? 2. I also find it better if we represent all the three types of constraints in the same way. Currently, the first uses gradient, the second uses Lie, and the third explicitly differentiates. A little messy. 3. I suggest presenting the optimization program before explaining its rationale. Instead, after presenting it formally, mention that constraints are really CBF constraints of the 3 types of safety specifications with some relaxation.}}, }\NA{Do we really need to refer to this paper here? We may get questions about what is new here.},% \NA{What are $\bar{\bff}$ and $\bar G$ below?} % \NA{There is overflow of the $\forall i$ terms below.}
The three inequalities correspond to the positional, higher-order, and velocity safety specifications, respectively. Because the safe set is a union of the component sets \(\calS_i\), not every component needs to be active at a given state. The term \(M\Delta_i(\bfx)\) relaxes the constraints associated with inactive components without adding or removing the constraints. If bubble $i$ is inactive, then \(\Delta_i(\bfx) > 0\) and multiplication with a large $M$ relaxes the constraint for $i$ as a large slack is added to the left-hand side. If $\Delta_i(\bfx) = 0$, i.e., bubble $i$ is active, then the constraint applies unmodified.
We refer to the QP in \eqref{eq:qp_cbf} as \textsc{BubbleCBF}.
% Since the all-component\NA{I would avoid using the term "all-component" because we have not defined what it means. I feel that this whole paragraph is not very useful. What is it explaining? That we need to add a term $\beta \alpha^2$ to the cost function? Why not just do that directly in (9) above and save a whole paragraph of vague explanations?} controller optimizes over \(\alpha\), we cannot directly prescribe relaxed or conservative CBF behavior through a fixed parameter. However, we can bias this behavior by including a weighted penalty on  \(\alpha\) (or on its deviation from a chosen reference/lower bound) in the QP cost matrix \(P\) in~\eqref{eq:qp_cbf}. With an appropriate reference, this encourages smaller \(\alpha\) values, which can cause the velocity CBF to activate more often, yielding slower but more controlled motion. This is especially relevant for small bubbles where the radius is smaller than 1,
% %\BL{what do we mean by fractional radii?}),
% where barrier scaling may produce steeper gradients and increase controller sensitivity.
% \BL{Move this to after the QP-CBF optimization problem~\eqref{eq:qp_cbf}.}
% Jointly optimizing $M$ and $\alpha$ provides the flexibility required by
% the feasibility result in~\cite{MA-NA-JC:24-auto} %\NA{We should avoid referring to prior work so much because it will raise questions about the novelty of this work.}.
% optimized jointly with $\bfu$ to improve feasibility while remaining bounded away from zero by $c_\alpha$.

\begin{figure*}
    \includegraphics[width=\linewidth,trim={0 0.0cm 0 0.0cm},clip]{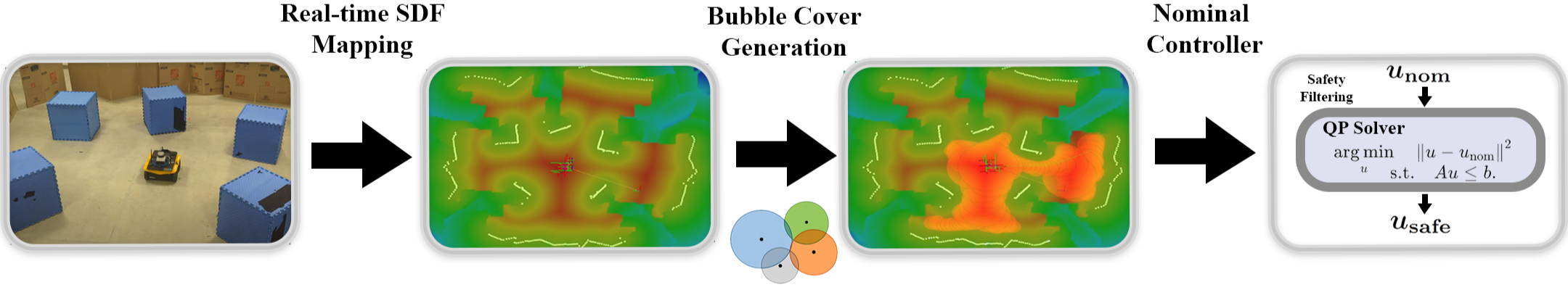}
    \caption{Application of \textsc{BubbleCBF} to differential-drive robot navigation. Using measurements from range sensors on the robot (1st panel), we estimate an SDF model of the obstacles (2nd panel) in real-time. We construct safe bubbles from the SDF (3rd panel) and define a nominal controller and constraints for the \textsc{BubbleCBF} optimization (4th panel).}
    %, which guarantees the forward invariance of the union of bubbles for a wide class of second-order systems.}
    \label{fig:overall_piepline}
\end{figure*}

\subsection{Feasibility and Safety Analysis}
\label{sec:analysis}
We now present conditions on the system dynamics and control design parameters that guarantee feasibility of the QP in \eqref{eq:qp_cbf} and forward invariance of the set $\calS$ in \eqref{eq:safe_set}. We place the following assumptions on the dynamics~\eqref{eq:second_order_dynamics}, which hold for Euler-Lagrange systems~\cite{RO-AL-PJN-HSR:98}.

\begin{assumption}\label{as:EL}
The dynamics
% \MMA{is 'dynamics singular?' I genuinely don't know :)}
~\eqref{eq:second_order_dynamics} satisfies the following:
    \begin{enumerate}[(a)]
    \item the function $G(\bfx)$ has right inverse $G^\dagger(\bfx)$ for all
      $\mathbf{x} \in \calX$;
  \item the right inverse $G^\dagger(\bfx)$ satisfies
    $\|G^\dagger(\bfx)\|_\infty \leq k_G$ for all $\bfx \in \calS$ and some $k_G < \infty$;
    \item the drift term $\mathbf{f}(\mathbf{x})$ is of the form $\mathbf{f}(\mathbf{x}) = \mathbf{f}_{q}(\mathbf{q}) + \mathbf{f}_{v}(\mathbf{v})$, where the velocity-dependent part satisfies $\|\mathbf{f}_v(\mathbf{v})\|_\infty \leq k_v\|\mathbf{v}\|_\infty$ for some $k_v >0$ and the position-dependent part satisfies $\|\mathbf{f}_q(\mathbf{q})\|_\infty \leq k_1$ for some $k_1 < \infty$ and all $\bfq \in \calQ$.
  \end{enumerate}
\end{assumption}

% \begin{assumption}\label{as:control}
% The nominal controller $\unom : \calX \to \real^m$ is a continuous function of $\bfx$.\NA{We can move this assumption directly in Problem 1, i.e., say "... with continuous nominal controller..."}
% \end{assumption}
Our feasibility construction requires intersecting bubbles to share a point strictly inside both. We exclude degenerate intersections that are tangent on the boundary, which can be enforced during generation.

\begin{assumption}[Bubble intersections have interior]\label{as:interior}
Let $S_{\cap}$ be the collection of all index sets with intersecting bubbles: if $\bigcap_{i \in \{i_1,\dots,i_\sigma\}}\calB_i \neq \emptyset$ then $\{i_1,\dots,i_\sigma\} \in S_\cap$.
For all $I \in S_{\cap}$, there exists a point $\bfy_I \in \bigcap_{i \in I}\interior{\calB_i}$.
%\NA{We should have this $\oprocend$ either everywhere or nowhere.}
\end{assumption}
%For the spherical bubbles we consider, 
% Assumption~\ref{as:interior} means no two bubbles touch only at their boundaries, which is easily enforced in practice.\NA{Move this explanation before the assumption as a way to introduce it.}

Pick an arbitrary $\bfy_I$ for every $I \in S_\cap$ and define the overlap size for every bubble $\calB_i$:
\begin{equation}\label{eq:def_delta_i}
  \delta_i^2 \coloneqq \max_{i \in I \in S_\cap} \|\bfc_i - \bfy_I\|^2.
\end{equation}
%The following lemma shows $\delta_i$ is strictly less than the radius of $\calB_i$.
%is well-posed \NA{What does "well-posed" mean?}

\begin{lemma}\label{lem:deltaPositive}
  For every bubble $\calB_i$, $\delta_i^2 < r_i^2$.
\end{lemma}

\begin{proof}
  By~\eqref{eq:def_delta_i}, there exists $\bfy_I \in \interior{\calB_i}$ with $\delta_i = \|\bfc_i - \bfy_I\|$. Since $\bfy_I \in \interior{\calB_i}$, we have $r_i^2 > \|\bfc_i - \bfy_I\|^2 = \delta_i^2$.
\end{proof}

% We are now ready to state the main result\NA{Explain at a high-level what the main result is before stating it.}.

We establish a safety guarantee for $\calS$ under the controller~\eqref{eq:qp_cbf}.
Under Assumptions~\ref{as:EL} and~\ref{as:interior}, the parameters of~\eqref{eq:qp_cbf} can be chosen so that the optimization problem in \eqref{eq:qp_cbf} is feasible for any state in $\calS$. In turn, the controller resulting from \eqref{eq:qp_cbf} is continuous and renders the set $\calS$ forward invariant.

\begin{theorem}[$\calS$ is control invariant]\label{thm:safe}
Consider the set $\calS$ defined in~\eqref{eq:safe_set}, and let $r_{\max}\coloneqq\max_{i\in\mathcal I}r_i$. Under Assumptions~\ref{as:EL} and \ref{as:interior}, if the parameters $v_{\max}$, $\gamma$, and $\epsilon_i$ satisfy $\gamma < \frac{v_{\max}}{r_{\max}}$ and $\frac{1}{2}\gamma(r_i^2 - \delta_i^2) > \epsilon_i$ for all $i$ (which is always possible by Lemma~\ref{lem:deltaPositive}), the QP in \eqref{eq:qp_cbf} is feasible for all $\bfx \in \calS$, its solution $\bfu^*$ is continuous in $\bfx$, and $\calS$ is forward invariant for the system in \eqref{eq:second_order_dynamics} with controller $\bfu = \bfu^*(\bfx)$.
%\NA{Note that $\bfu^*$ is different from $\bfu^*$. We should use consistent notation.}.
\end{theorem}
%
% \marginJC{Shouldn't we say in the theorem statement that $\mathcal{S}$ as defined in (6)?}
%
% \MA{stated $\mathcal{S}$ and also the continuity that is implied from Thm 4.12 of \cite{MA-NA-JC:24-auto}}

\begin{proof}[Proof sketch]
%\NA{I don't find this proof sketch intuitive or easy to read. Then, what is the point of it? Perhpas, we can directly refer to the full proof in Appendix~\ref{app:proof_safe}.}
The full proof in Appendix~\ref{app:proof_safe} relies on \cite[Theorem~4.12]{MA-NA-JC:24-auto}, which shows that the existence of a control input satisfying the constraints of \eqref{eq:qp_cbf} at each point on the safe-set boundary $\partial\calS$ is sufficient to establish feasibility of \eqref{eq:qp_cbf}, continuity of its optimal solution in the state, and forward invariance of $\calS$.
%
% \marginJC{"sufficient for the feasibility of \eqref{eq:qp_cbf} and hence forward invariance". How about continuity?}
%
We show that the following controller satisfies the constraints:
\begin{equation}\label{eq:thm_feasible}
\bfu(\bfx) = -k\, G^{\dagger}(\bfx)\!\left(\bfv + \tfrac{1}{2}\gamma (\bfq - \bfy_{\Ic(\bfx)})\right),
\end{equation}
where $k > 0$ is a sufficiently large gain, $\Ic(\bfx) \coloneqq \{i \mid s_i(\bfx)=s(\bfx)\}$ is the active index set, and $\bfy_{\Ic(\bfx)}$ is the corresponding interior point that exists by Assumption~\ref{as:interior}, with $\bfy_{\emptyset} = \bfq$.
%The proof verifies that~\eqref{eq:thm_feasible} satisfies all QP constraints at the boundary of $\calS$.
\end{proof}

Although \eqref{eq:qp_cbf} is feasible and its solution renders $\calS$ forward invariant, the magnitude of the required control inputs may be unbounded because the proof relies on a sufficiently large gain $k > 0$. This may not be physically possible. Therefore, we next derive a sufficient condition on the available control authority that preserves feasibility under bounded inputs.

\begin{theorem}[Input limits]\label{thm:inputConstraints}
  Let $a_1$ and $a_2$ denote the gain thresholds defined in
  \eqref{eq:aCond1} and \eqref{eq:aCond2}. Let the allowable control set be constrained as $\Uc \supseteq  \{\bfu \;|\; \|\bfu\| \leq u_{\max}\}$. If 
\begin{equation}\label{eq:inputConstraints}
u_{\max}
>
k_G(v_{\max}+\gamma r_{\max})\max\left\{
a_1,a_2
\right\}.
\end{equation}
then the QP in \eqref{eq:qp_cbf} remains feasible.
% with the additional constraint $\|\bfu\| \leq u_{\max}$.
% \NA{show explicitly what the additional constraint is.}.
% for all $i$, then $v_{\max}$\BL{As what? Add the condition here. I forget what this was.} can be chosen so that the controller \eqref{eq:qp_cbf} is feasible with the additional constraint $\bfu \in \Uc$. 
\end{theorem}

\begin{proof} 
See Appendix~\ref{app:proof_bound}
\end{proof}

Theorem~\ref{thm:inputConstraints} gives a sufficient control-authority condition for feasibility of the optimization problem in \eqref{eq:qp_cbf}. The bound in \eqref{eq:inputConstraints} relates the admissible input magnitude to the system constants and the design parameters $\gamma$ and $v_{\max}$. It can be used either to assess whether given actuator limits are sufficient or to guide parameter selection under fixed input constraints. 

%Since the condition is sufficient, lower $u_{\max}$ may also be compatible.
% [Differential-drive running example analysis moved to the Application section (Sec.~\ref{sec:application})]

%% file: tex2/planning.tex
%
% \marginJC{But in the previous section, we never introduced the term \textsc{BubbleCBF}. We should have a sentence either at the beginning or end of Section IV that says "this is what we refer to as \textsc{BubbleCBF}".}
%

\section{Safe Differential-Drive Robot Navigation}
\label{sec:application}

We apply \textsc{BubbleCBF} to a differential-drive robot navigating in an unknown environment. Fig.~\ref{fig:overall_piepline} presents an overview of our approach. An online SDF estimate, which provides the distance from query points in the workspace to the nearest obstacles, is used to construct the union of bubbles \eqref{eq:positional_safety_requirement} with Algorithm~\ref{alg:bubble_graph}. 
% The bubble cover is used to construct a Bézier reference trajectory by first selecting a shortest path through the graph of overlapping bubbles and then optimizing a continuous trajectory within the selected bubble sequence \cite{lee2024safe}. 
% This reference trajectory is tracked by the nominal controller \(\unom\). 
These safety constraints are enforced by the \textsc{BubbleCBF} safety filter.

%\NA{Rewrite this sentence. Explain what a "reference path" is, how it is constructed, and how it is related to the nominal controller. If referring the BubbleCBF, use a consistent font.}.

% We first define the robot dynamics. We then describe how collision-free bubbles can be constructed online from the estimated SDF. Finally, we verify that the assumptions required by the proposed method are satisfied.\BL{this para is a bit repetitive.}

%\subsection{Differential Drive Robot}

%\NA{Avoid switching the naming between "differential-drive" and "unicycle". Choose one and stick with it but do not present both. Also, at this point, since we switch the presentation to be the control design first and then the application, we do not need to present this as an example. Directly talk about the application to autonomous robot navigation and introduce the robot dynamics.}
%\MA{Addressed, is "Definition" better?}

%\begin{definition}[Differential-drive model]\label{ex:diffdrive}
%We consider a robot described by differential-drive kinematics:

%Dynamics of the form~\eqref{eq:second_order_dynamics} arise in many mechanical systems. 
We consider a differential-drive robot with kinematics:
\begin{equation} \label{eq:diffdrive1}
    \begin{bmatrix}
        \dot{x} \\
        \dot{y} \\
        \dot{\theta}
    \end{bmatrix} =
    \begin{bmatrix}
        V \cos\theta \\
        V \sin\theta \\
        \omega
    \end{bmatrix},
\end{equation}
where $x$, $y$, and $\theta$ are the robot's position and heading, and $V$ and $\omega$ are the linear and angular speeds. By differentiating the translational kinematics in \eqref{eq:diffdrive1}, the model can be written in the second-order form of \eqref{eq:second_order_dynamics}:
%
% \marginJC{"eliminating" sounds odd to me. You mean that we express $\theta$ as a function of $v_x$ and $v_y$?}
% %
% the differential-drive model in \eqref{eq:diffdrive1} can be written in the second-order form of \eqref{eq:second_order_dynamics}:
%
\begin{equation}\label{eq:diffdrive2}
\begin{bmatrix}
\dot{v}_x\\
\dot{v}_y
\end{bmatrix}
=
\begin{bmatrix}
a \cos \theta - \omega \sqrt{v_x^2 + v_y^2}\, \sin \theta\\
a \sin \theta + \omega \sqrt{v_x^2 + v_y^2}\, \cos \theta
\end{bmatrix},
\;\; \theta = \operatorname{atan2}(v_y,\, v_x),
\end{equation}
with generalized position $\bfq = \begin{bmatrix} x & y \end{bmatrix}^\top$, generalized velocity $\bfv = \begin{bmatrix} v_x & v_y \end{bmatrix}^\top$, and control input $\bfu = \begin{bmatrix} a & \omega \end{bmatrix}^\top$ consisting of linear acceleration and angular speed. 
%\end{definition}
% 
%\NA{I don't think this section needs subsections.}

%\subsection{Bubble Cover in Unknown Environments}
% Lee et. al. ~\cite{KMBL-ZD-CLG-LW-NA-TVC:24} have shown that the free space can be modeled as a union of safe bubbles with radius set to the signed distance value at the center's location, as shown in the following definition.
The robot operates in a static environment with obstacles $\mathcal{O} \subset \mathcal{Q}$.
%\NA{Introduce the obstacles earlier, i.e., introduce the objective and robot dynamics, obstacles, and sensing in the beginning of this section.}.
Using a range sensor and an algorithm like Kernel-SDF \cite{kernelSDF} or iSDF \cite{isdf}, we estimate the workspace SDF:
% \NA{We should give examples of algorithms that are used to estimate SDF.} maintains an estimated signed distance field (SDF):
\begin{equation}\label{eq:distance_field}
d_t(\mathbf{q}) =
\begin{cases}
\;\;\;\min_{\mathbf{q}' \in \partial \mathcal{A}_t}\|\mathbf{q}-\mathbf{q}'\|, & \mathbf{q} \notin \mathcal{A}_t,\\[4pt]
-\min_{\mathbf{q}' \in \partial \mathcal{A}_t}\|\mathbf{q}-\mathbf{q}'\|, & \mathbf{q} \in \mathcal{A}_t,
\end{cases}
\end{equation}
where $\mathcal{A}_t \coloneqq \mathcal{O}_t \cup \mathcal{Z}_t$, $\mathcal{O}_t \subseteq \mathcal{O}$
% \NA{I am not sure that it is important to emphasize that $\calO$ is an estimate by using a hat. Why don't we have a hat on $\calU$? Do we do something about the error in our design? If we do not address the error in any way, it is not necessary to emphasize the hats. Even better, it would be fine to not use hats but mention somewhere that we can use the parameters in the CBF definition to tighten the safety constraints and account of estimation errors in the SDF mapping or robot localization.} 
is the observed obstacle set up to time $t$, and $\mathcal{Z}_t$ is the unobserved space (treated as occupied). Each SDF value induces a safe bubble of the form defined in \eqref{eq:positional_safe_ball}. At query point \(\bfq_i\), the SDF $d_t(\bfq_i)$ gives the distance to the nearest obstacle boundary, so setting $r_i \leq d_t(\bfq_i)$ yields an obstacle-free bubble. Since real-world SDF estimates are subject to uncertainty, the safety margin \(\nu\) can be chosen adaptively based on the estimator variance. In particular, we set $ 
\nu=\sqrt{\frac{1-p}{p}}\,\sigma_t(\mathbf{q}),
$ where \(\sigma_t(\mathbf{q})\) is the standard deviation of the SDF given by Kernel-SDF~\cite{kernelSDF} and \(p\) is a chosen risk probability. This ensures that the corresponding bubble is safe with probability at least \(1-p\) by Cantelli's inequality. Thus, the safety margin in Algorithm \(\ref{alg:bubble_graph}\) can be adjusted to account for sensor uncertainty.

To generate a collection of safe bubbles, we adapt
%\NA{What does "extend the application of" mean. Write this in a simpler way.} 
the rapidly-exploring bubble graph (RBG) algorithm of~\cite{lee2024safe} to static unknown environments. 
%The RBG algorithm generates a collection of such bubbles. 
%As shown in Alg.~\ref{alg:bubble_graph},
% \NA{The notation in the algorithm is not consistent. It switches between $\mathcal{B}$ and $B$. Also, neither $\mathcal{B}(\bfq,r)$ nor $B(\bfq,r)$ is defined. The use is in conflict with eq. (2).}, 
The RBG algorithm (Algorithm~\ref{alg:bubble_graph}) iteratively samples a random point, finds the closest bubble, and extends a new bubble toward the random point, similar to rapidly exploring random trees \cite{LaValle1998RapidlyexploringRT}. 

\begin{algorithm}[t]
\caption{Rapidly Exploring Bubble Graph}
\label{alg:bubble_graph}
\small
\begin{algorithmic}[1]
\Require Start configuration $\mathbf{q}_{0}$, SDF $d(\cdot)$, safety margin $\nu$, max. number of bubbles $N_{\mathrm{max}}$, min. radius $r_{\mathrm{min}}$.
\Ensure Collection $\mathfrak{B}$ of safe bubbles covering the free space.

\State $r_0 \gets d(\mathbf{q}_0)-\nu$
\State $\mathcal{B}_0 \gets \{\mathbf{q}\in\mathcal{Q}\mid \|\mathbf{q}-\mathbf{q}_0\|^2 \leq r_0^2\}$
\State $\mathfrak{B} \gets \{\mathcal{B}_0\}$

\While{$|\mathfrak{B}| < N_{\mathrm{max}}$}
    \State $\mathbf{q}_{\mathrm{rand}} \gets \Call{SampleRandom}{}$
    \State $\mathcal{B}_{\mathrm{near}} \gets \Call{NearestBubble}{\mathfrak{B}, \mathbf{q}_{\mathrm{rand}}}$

    \If{$\mathbf{q}_{\mathrm{rand}} \notin \mathcal{B}_{\mathrm{near}}$}
        \State $\mathbf{q}_{\mathrm{new}} \gets \Call{ExtendToward}{\mathcal{B}_{\mathrm{near}}, \mathbf{q}_{\mathrm{rand}}}$
        \State $r_{\mathrm{new}} \gets d(\mathbf{q}_{\mathrm{new}}) - \nu$
        \If{$r_{\mathrm{new}} > r_{\mathrm{min}}$}
            \State $\mathcal{B}_{\mathrm{new}} \gets \{\mathbf{q}\in\mathcal{Q}\mid \|\mathbf{q}-\mathbf{q}_{\mathrm{new}}\|^2 \leq r_{\mathrm{new}}^2\}$
            \State $\mathfrak{B} \gets \mathfrak{B} \cup \{\mathcal{B}_{\mathrm{new}}\}$
        \EndIf
    \EndIf
\EndWhile

\State \Return $\mathfrak{B}$
\end{algorithmic}
\end{algorithm}

% \BL{Also need to add the trajectory optimization? Maybe in experiments.}
In a static unknown environment, treating unobserved space as unsafe makes SDF values non-decreasing as the map is updated. Thus, the computed bubbles remain safe even if their radii change. The following result formalizes this property.

% we avoid per-step queries because the initially computed bubble chain remains certifiably safe. 
% Even though the SDF values are queried at a particular time $t$, the bubbles remain safe at all other times in the future, assuming a static environment.\BL{merge sentence}

\begin{proposition}\label{prop:unknown_space_bubble}
Consider bubbles generated using an SDF of $\mathcal{A}_t$. The following holds:
\begin{enumerate}
    \item for any $t$ and $\bfq \in \calQ$, $\calB_{t}(\bfq) \subset (\calQ \setminus \calO)$;
    \item for any $t_2 \geq t_1$, and any $\bfq \in \calQ$, $\calB_{t_1}(\bfq) \subseteq \calB_{t_2}(\bfq)$;
    \item for any $t_2 \geq t_1$ and any $\bfq_1, \bfq_2 \in \calQ$, $\calB_{t_1}(\bfq_1) \cap \calB_{t_1}(\bfq_2) \neq \emptyset$ implies $\calB_{t_2}(\bfq_1) \cap \calB_{t_2}(\bfq_2) \neq \emptyset$.
\end{enumerate}
\end{proposition}
% \BL{State this as a formal statement, and in the later paragraph explain how this relates to the bubble cover.}
% \ZD{I am confused. Are bubbles shrinking $\calB_{t_2}(\bfq) \subseteq \calB_{t_1}(\bfq)$ or expanding $d_{t+1}(\mathbf{q}) \ge d_t(\mathbf{q}), \quad \forall \mathbf{q}$.}
% \MA{Zhirui, fixed}

\begin{proof}
% \BL{Need to refine}
% \MA{refined I believe}
% The statements follow from the fact that the distance field, and hence the radii of bubbles, is non-decreasing for all \(\mathbf{q}\), \(d_{t+1}(\mathbf{q}) \ge d_t(\mathbf{q})\).
% The non-decreasing property can be seen by noting that, since \(\mathcal{U}_{t+1}\subseteq \mathcal{U}_t\), we have \(\mathcal{A}_{t+1}\subseteq \mathcal{A}_t\). Therefore, by the definition of the distance field as a minimum over \(\mathcal{A}_t\),
% \[
% d_{t+1}(\mathbf{q}) \ge d_t(\mathbf{q}), \qquad \forall \mathbf{q}.
% \]

% Now if adjacent bubbles initially overlap,
% \[
% \|\mathbf{q}_i-\mathbf{q}_{i+1}\| \le r_i^t+r_{i+1}^t,
% \]
% then
% \[
% \|\mathbf{q}_i-\mathbf{q}_{i+1}\| \le r_i^{t+1}+r_{i+1}^{t+1},
% \]
% so overlap, and thus connectivity, is preserved.
% Safety is also preserved since each radius is given by the SDF value. Therefore, the initial bubble cover remains safe and connected as the SDF updates.
Since $\mathcal{Z}_{t+1}\subseteq \mathcal{Z}_t$, we have $\mathcal{A}_{t+1}\subseteq \mathcal{A}_t$, and therefore
\[
d_{t+1}(\mathbf{q}) \ge d_t(\mathbf{q}), \qquad \forall \mathbf{q}.
\]
Hence, the bubble radii are non-decreasing, $r_i^{t+1}\ge r_i^t$. Thus, if adjacent bubbles overlap at time $t$, i.e., $\|\mathbf{q}_i-\mathbf{q}_{i+1}\| \le r_i^t+r_{i+1}^t$, then $\|\mathbf{q}_i-\mathbf{q}_{i+1}\| \le r_i^{t+1}+r_{i+1}^{t+1}$ as well, so connectivity is preserved. Safety is also preserved since each radius is given by the SDF. Therefore, the initial bubble cover remains safe and connected under SDF updates.
\end{proof}

Proposition~\ref{prop:unknown_space_bubble} implies that a bubble cover that is initially safe and connected remains safe and connected under subsequent SDF updates. Consequently, the robot can traverse the original bubble cover without repeatedly querying the SDF.

%\subsection{Feasibility Analysis for Differential-Drive Dynamics}

%We examine the assumptions and analyses in the example of differential-drive robots in~\eqref{eq:diffdrive}. 
%First, the dynamics in~\eqref{eq:diffdrive} violates Assumption~\ref{as:EL}(a) when $\sqrt{v_x^2 + v_y^2} = 0$ as this is the gain for $\omega$.

% If the system state lies in the safe set, a speed equal to zero keeps the state within the safe set. Therefore, although the CBF condition formally requires $\|\bfv\| \neq 0$, the case $\|\bfv\| = 0$ does not compromise safety since the system remains in the safe set.

Next, we examine the feasibility of the controller \eqref{eq:qp_cbf} for differential-drive robots.

\begin{corollary}[Differential-drive safe control feasibility under input bounds]
\label{thm:diffdrive_bounds}
For the differential-drive model in \eqref{eq:diffdrive2}, let
\[
v_{\max}=\eta\gamma r_{\max},\qquad
\epsilon_i=\rho_i\frac{\gamma}{2}(r_i^2-\delta_i^2),
\]
with $\eta>1$ and $\rho_i\in(0,1)$. Fix any $0<\beta<\eta$ and let
$v_{\min}:=\beta\gamma r_{\max}$. Then, for any $u_{\max}>0$ and
sufficiently small $\gamma>0$, the QP in \eqref{eq:qp_cbf}, with
$\bfu\in\Uc$, is feasible at states satisfying
$\|\bfv\|> v_{\min}$.
\end{corollary}

\begin{proof}
See Appendix~\ref{app:diffdrive_bounds}.
\end{proof}
An important requirement is $\|\bfv\|>v_{\mathrm{min}}>0$. At $\|\bfv\|=0$, $G$ loses rank and $\|G^\dagger(\bfx)\|_\infty$ becomes unbounded, which can lead to infeasibility because it violates the assumption on bounded $k_G$. In practice, we recover from low speed by turning in place toward the nearest center of a bubble containing the robot, then accelerating toward it. The straight-line segment stays inside the bubble, so this motion remains collision free.

%% file: tex2/results.tex
\section{Experiments}\label{sec:experiments}

We evaluate whether \textsc{BubbleCBF} can maintain safe, goal-directed navigation on a differential-drive robot as nominal tracking degrades. We compare it against the unfiltered nominal controller and a first-order SDF-CBF baseline \cite{CBFdeadlock}. In simulation, we isolate the effect of increasing reference speed. In real-world experiments, we evaluate the complete navigation system on an exploration task in an initially unknown environment, where the robot builds a map online and follows generated reference paths to expand its coverage during a two-minute interval.

Across both studies, each trial is assigned one of four mutually exclusive outcomes: arrived, collided, stuck, or timed out. A trial is successful if the robot reaches the goal and is classified as a deadlock if it remains stationary until a time budget of 60 seconds expires. A trial that remains in motion without reaching the goal is classified as a timeout. We report the success rate (SR), collision rate (CR), deadlock rate (DR), and timeout rate (TR). For the real-world experiments, we introduce two additional metrics that are relevant to the exploration objective, covered area (CA) to account for the area covered by the bubble cover and average distance to the closest obstacle (OD).

% \NA{Should we provide an overview of how this section is organized?}
% Our experiments are sectioned into two categories, simulation and real world experiments. 

\subsection{Simulated Autonomous Navigation}
The nominal controller combines pure pursuit steering with proportional speed
control, with fixed gains $\beta_v$ and $\beta_{\theta}$, and  the current lookahead point $[x_L,y_L]^\top$ across all experiments.
Increasing $v_{\mathrm{des}}$ therefore tests performance under more demanding
tracking conditions. The nominal control is given by \begin{equation}
\label{eq:nominal}
\bfu_{\mathrm{nom}} =
\begin{bmatrix}
  \beta_v\,\dfrac{v_{\mathrm{des}} - \|\mathbf v\|}{\Delta t} \\
  \beta_\theta\,\!\bigl(\operatorname{atan2}(y_L - y,\; x_L - x) - \theta\bigr)
\end{bmatrix}.
\end{equation} Both safety filters receive the same nominal input from the controller~\eqref{eq:nominal}.
The first filter is the proposed bubble filter \eqref{eq:qp_cbf}. The second filter is an SDF-CBF baseline with barrier
$d(\bfq)-\epsilon_d$, where $d$ is assumed continuously differentiable
and $\epsilon_d>0$ is a safety margin. It uses the kinematic
unicycle
$\dot{\bfq}=[V\cos\theta,V\sin\theta,\omega]^\top$
with input $\bfu_d=[V,\omega]^\top$. We convert the nominal acceleration
into a speed command over one control step $\Delta t$, giving
$\bfu_{d,\mathrm{nom}}=[\|\bfv\|+a_{\mathrm{nom}}\Delta t,\;\omega_{\mathrm{nom}}]^\top$.
The filtered input is given by
\begin{equation}
\label{eq:qp_sdf}
\begin{aligned}
\min_{\bfu_d}\quad
& \|\bfu_d-\bfu_{d,\mathrm{nom}}(\bfx)\|^{2} \\
\text{s.t.}\quad
& \nabla d(\bfq)^\top \bfe_\theta V
  + \gamma_d\bigl(d(\bfq)-\epsilon_d\bigr) \ge 0,
\end{aligned}
\end{equation}
where $\bfe_\theta=[\cos\theta,\sin\theta]^\top$ and $\gamma_d>0$
is the CBF gain.
The filtered speed $v^*$ is mapped back to the plant input as
$a=(v^*-\|\bfv\|)/\Delta t$. The constraint involves only $v$, so
$\omega^*=\omega_{\mathrm{nom}}$. In effect, the baseline can slow the robot down but cannot
steer it away from obstacles. We evaluate $v_{\mathrm{des}}\in\{1.0,2.5,5.0\}$~m/s over 100 trials per speed, using a common goal and identical randomly sampled initial states across all methods. As shown in Table~\ref{tab:toy_exp_outcomes}, the performance gap widens as the desired speed increases. At $5$~m/s, the nominal controller collides in 62 trials, while SDF-CBF deadlocks in 19. In contrast, \textsc{BubbleCBF} remains collision- and deadlock-free in all 100 trials. The remaining \textsc{BubbleCBF} timeouts occur because the nominal command overshoots the goal despite remaining within the safe set. Overall, these results demonstrate \textsc{BubbleCBF}'s robustness to increasing tracking difficulty. Fig.~\ref{fig:toy_exp_outcomes} illustrates representative trajectories from one such scenario.

\begin{figure*}
    \centering
    \includegraphics[width=\linewidth]{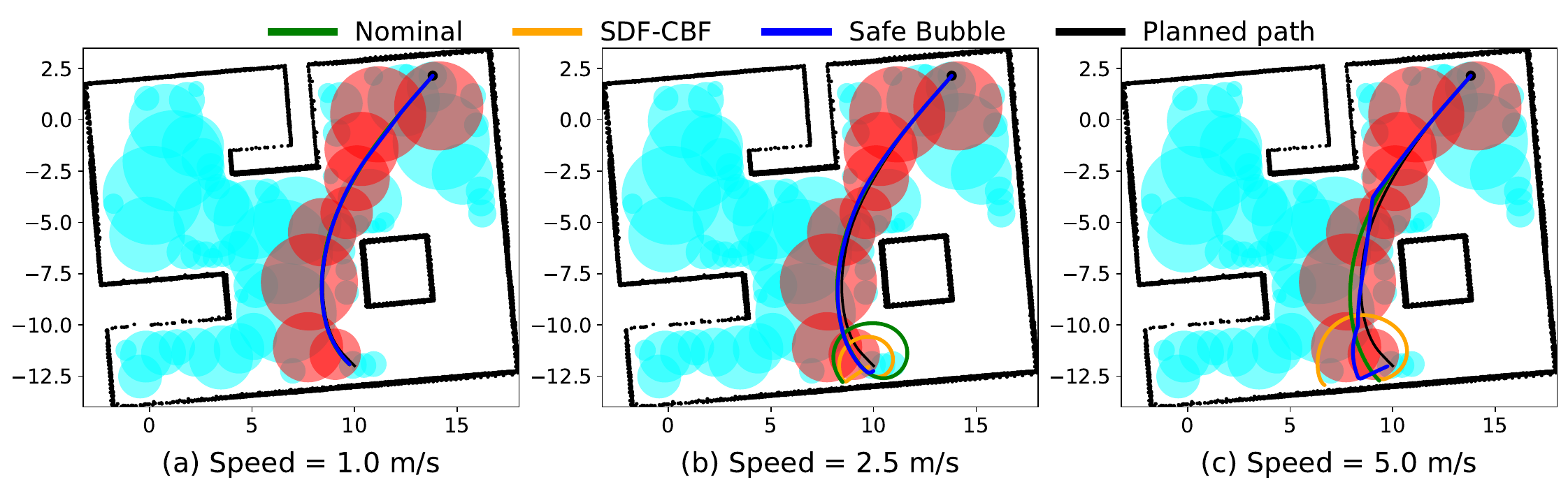}
\caption{All planners share the same start and goal positions. Cyan denotes the bubble cover, red the path bubbles, and black the planned trajectory. The nominal controller collides at speeds above $1$~m/s (panels (b) and (c)), while SDF-CBF deadlocks at $5$~m/s (panel (c)).}
    \label{fig:toy_exp_outcomes}
\end{figure*}

\begin{figure}
    \centering
    \includegraphics[width=\linewidth]{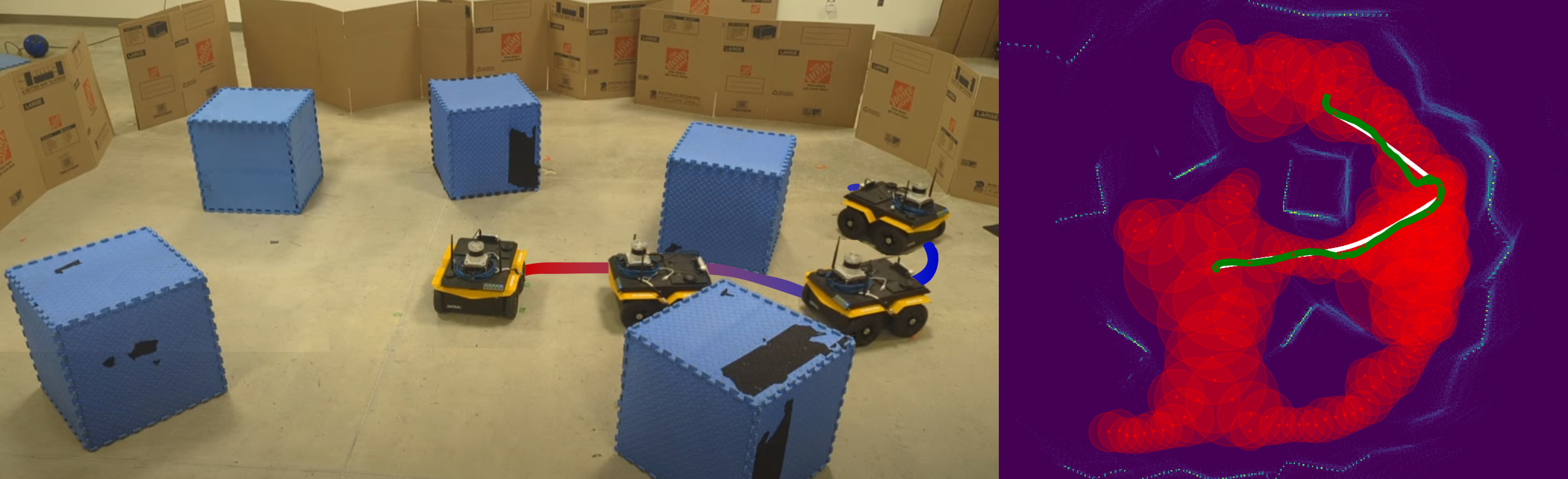}
    \caption{Navigation experiment in an a priori unknown environment. Left: the robot and cluttered test environment. Right: the generated bubble cover, with the executed robot trajectory in green and the reference B\'ezier  trajectory in white.}
    \label{fig:experimentationenv}
\end{figure}

\subsection{Real-World Autonomous Navigation}
We evaluate navigation on a Clearpath Jackal differential-drive robot equipped
with an Ouster OS1 LiDAR in the initially unknown cluttered environment shown
in Fig.~\ref{fig:experimentationenv}. Direct LiDAR-inertial odometry
(DLIO)~\cite{chen2022dlio} provides state estimates using the OS1's point cloud
and built-in IMU, while Kernel-SDF~\cite{kernelSDF} estimates the SDF online
from LiDAR observations. The robot’s linear and angular speeds are constrained to the Jackal’s maximum allowable values.

The RBG procedure in Algorithm~\ref{alg:bubble_graph} constructs a graph of
overlapping safe bubbles from the current SDF. An exploration target is
selected from the frontier bubbles of this cover. A shortest graph path
to the target determines a sequence of bubbles, and the reference trajectory
is generated as a B\'ezier curve constrained to lie within their union.
The bubble cover defines the CBF safe set. We run six trials per method and report SR, CR, DR, CA, and OD.
Each trial has a two-minute exploration budget and is successful if the
robot completes this interval without collision or deadlock. TR is not
reported because reaching the time limit marks completion of the exploration
trial rather than failure to reach a goal.
Table~\ref{tab:realworldres} summarizes the real-world results.
\textsc{BubbleCBF} completed all trials without collision or deadlock.
SDF-CBF remained collision-free but deadlocked in 50\% of the trials, and the nominal controller collided in 83.3\% of the trials. 

The observed difference in deadlock behavior is consistent with the structure
of the two safety filters. Under differential-drive kinematics, the first-order SDF-CBF constraint does not involve angular velocity $\omega$, so the filter
modifies translational speed while leaving the nominal steering command
unchanged. When the robot points toward an obstacle, the constraint restricts
forward speed as clearance decreases and can bring the robot to a stop at
the safety margin. If the nominal
steering command does not reorient the robot toward a feasible direction
of motion, the robot can remain stopped. Such deadlocks have also been reported in~\cite{CBFdeadlock}. In contrast, the smooth bubble barriers yield a higher-order constraint that includes both linear acceleration and angular
velocity, allowing \textsc{BubbleCBF} to also 
modify steering.

%  while continuing toward the goal

% % Requires: \usepackage{multirow}

% \begin{table}[t]
% \caption{Simulation performance across increasing reference speeds. Success, collision, and deadlock counts are reported over 10 trials for each method and speed. BubbleCBF maintains collision-free, deadlock-free navigation across all tested speeds.}
% \label{tab:toy_exp_outcome}
% \centering
% \small
% \renewcommand{\arraystretch}{1.0}
% \begin{tabular}{c|l|c|c|c}
% \hline
% \textbf{$v$} & \textbf{Method} & \textbf{Success} & \textbf{Collision} & \textbf{Deadlock} \\
% \hline
% \multirow{3}{*}{5.0}
%   & Nominal     & 5  & 5 & 0 \\
%   & BubbleCBF & \bf{10} & 0 & 0 \\
%   & SDF-CBF     & 8  & 0 & 2 \\
% \hline
% \multirow{3}{*}{2.5}
%   & Nominal     & 7  & 3 & 0 \\
%   & BubbleCBF & \underline{10} & 0 & 0 \\
%   & SDF-CBF     & \underline{10} & 0 & 0 \\
% \hline
% \multirow{3}{*}{1.0}
%   & Nominal     & \underline{10} & 0 & 0 \\
%   & BubbleCBF & \underline{10} & 0 & 0 \\
%   & SDF-CBF     & \underline{10} & 0 & 0 \\
% \hline
% \end{tabular}
% \end{table}
\begin{table}[t]
\caption{Simulation performance at increasing reference speeds. SR,
CR, TR, and DR are reported over 100 randomly sampled
initial positions for each method and speed. navigation at all tested speeds.}
\label{tab:toy_exp_outcomes}
\centering
\small
\renewcommand{\arraystretch}{1.0}
\begin{tabular}{c|l|cccc}
\hline
\textbf{$v$} & \textbf{Method} &
\textbf{SR (\%)} & \textbf{CR (\%)} &
\textbf{DR (\%)} & \textbf{TR (\%)} \\
\hline
\multirow{3}{*}{5.0}
  & Nominal   & 38  & 62 & 0 &0 \\
  & BubbleCBF & \textbf{100} & \textbf{0} & \textbf{0} &0\\
  & SDF-CBF   & 81  & \textbf{0} & 19 &0\\
\hline
\multirow{3}{*}{2.5}
  & Nominal   & 84  & 16 & 0 &0\\
  & BubbleCBF & \textbf{96} & \textbf{0} & \textbf{0} & 4 \\
  & SDF-CBF   & 87  & \textbf{0} & 13&0 \\
\hline
\multirow{3}{*}{1.0}
  & Nominal   & 95  & 5 & 0 &0\\
  & BubbleCBF & \textbf{100} & \textbf{0} & \textbf{0} &0\\
  & SDF-CBF   & 98  & \textbf{0} & 2 &0\\
\hline
\end{tabular}
\end{table}

\begin{table}[t]
\centering
\caption{Results for real-world experiments. CR, SR, DR, CA, and OD from obstacles. Start goals were identical, each planner had an exploration budget of 2 minutes. Success is defined when the robot does not experience a collision or deadlock in the time budget.}
\label{tab:realworldres}
\setlength{\tabcolsep}{5pt}
\renewcommand{\arraystretch}{1.15}
\begin{tabular}{lccc}
\hline
\textbf{Metric} & \textbf{BubbleCBF} & \textbf{Nominal} & \textbf{SDF CBF} \\
\hline
CR (\%) $\downarrow$ & \textbf{0.0} & 83.3 & \textbf{0.0} \\
SR (\%) $\uparrow$  & \textbf{100.0} & 16.7 & 50.0 \\
DR (\%) $\downarrow$ & \textbf{0.0} & \textbf{0.0} & 50.0 \\
CA (m$^2$) $\uparrow$       & 8.42 $\pm$ 0.89 & 7.86 $\pm$ 1.75 & \textbf{8.75 $\pm$ 0.80} \\
OD (m) $\uparrow$ & \textbf{0.720 $\pm$ 0.030} & 0.714 $\pm$ 0.052 & 0.687 $\pm$ 0.028 \\
\hline
\end{tabular}
\end{table}
We further evaluate \textsc{BubbleCBF} in a larger lobby environment, constructing the SDF and bubble cover online during navigation. Fig.~\ref{fig:env_views} shows the resulting trajectories and bubble cover, with performance summarized in Table~\ref{tab:largeworldres}.

% \subsection{Robustness to SDF service performance}
% In this section, we showcase that the BubbleCBF is able to show robustness to how fast the SDF service can return values. Since the bubbles are saved as a simple geometric primitive that can be loaded into the QP as constraints, it does not need to be called at each control loop. This decouples the control loop with perception updates. However, a typical SDF-CBF will need the gradient and the SDF value at each control iteration. The SDF service can prove as a bottleneck. Figure~\ref{fig:CBF_control_rate} showcases the robustness of the BubbleCBF over SDF-CBF when it comes to maintaining high control frequency.
% \begin{figure}
%     \centering
%     \includegraphics[width=1\linewidth]{toyResults/SDF Speed Comparison/SDF_speed_comparison.png}
%     \caption{Comparisson of final control frequency between the BubbleCBF and SDF-CBF}
%     \label{fig:CBF_control_rate}
% \end{figure}
\begin{figure}[t]
    \centering
    % \includegraphics[width=0.32\textwidth]{RealEnvFigs/Env1_cropped_same_size.png}\hfill
    % \includegraphics[width=0.32\textwidth]{RealEnvFigs/Env2_cropped_same_size.png}\hfill
    % \includegraphics[width=0.32\textwidth]{RealEnvFigs/Env3_cropped_same_size.png}
    % \caption{Robot platform in the indoor environment from three viewpoints.}
    \includegraphics[width=\linewidth,trim={0 5.5cm 0 6cm},clip]{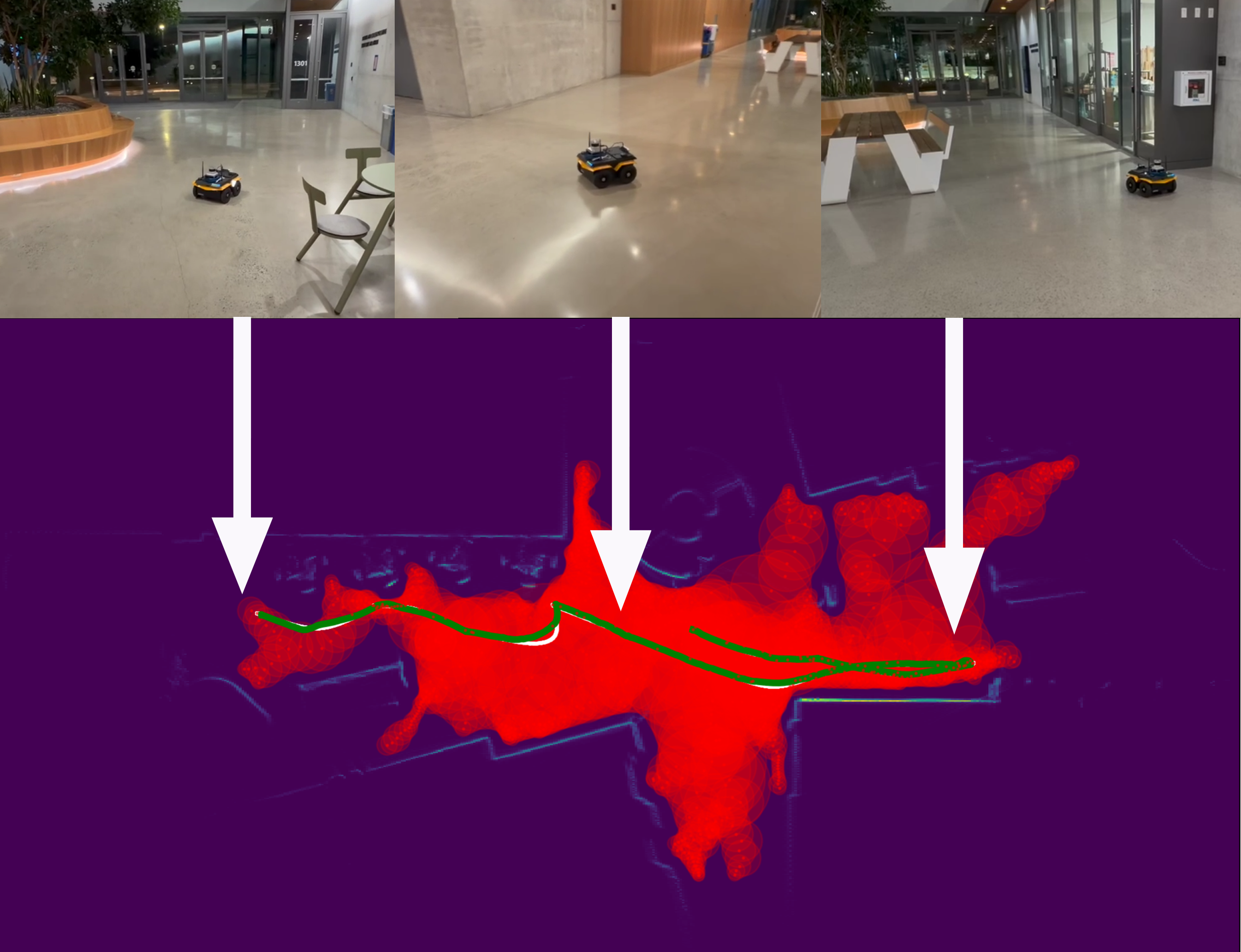}
    \caption{Autonomous navigation in a large unknown environment. Top: the robot at three times during the experiment. Bottom: the online-constructed bubble cover (red), reference B\'ezier trajectory (white), and executed trajectory (green). White arrows indicate the robot positions shown in the top panels.}
    \label{fig:env_views}
\end{figure}
%
% \marginJC{What do the large white arrows denote in Fig 5?}
%

\begin{table}[t]
\centering
\caption{Performance of \textsc{BubbleCBF} in the large-scale unknown environment. CR, SR, DR, CA, and OD from obstacles.}
\label{tab:largeworldres}
\setlength{\tabcolsep}{4pt}
\renewcommand{\arraystretch}{1.15}
\begin{tabular}{ccccc}
\hline
\textbf{CR (\%) $\downarrow$} & \textbf{SR (\%) $\uparrow$} & \textbf{DR (\%) $\downarrow$} & \textbf{CA (m$^2$) $\uparrow$} & \textbf{OD (m) $\uparrow$} \\
\hline
0.0 & 100.0 & 0.0 & $126.99 \pm 43.13$ & $1.247 \pm 0.31$ \\
\hline
\end{tabular}
\end{table}

\par 
% An extension to dynamic enviornments can be formulated as updating the bubble radii by querying the SDF online and rebuilding the CBF optimization, this way every discrete timestep would account for local safety gaurantees within the time step and 

%% file: tex2/conclusion.tex
\section{Conclusions}
We presented \textsc{BubbleCBF}, a safe control synthesis method that converts spherical primitives (bubbles) obtained from signed distance measurements into a feasible higher-order CBF for second-order systems.
%an arbitrary SDF representation into collision-free spherical primitives and uses their quadratic geometry to define a feasible higher-order CBF for second-order systems. 
We proved that a quadratic program with our second-order CBF constraint yields a continuous controller, renders the bubble set forward invariant, and remains feasible under sufficient bounded-input conditions. Simulation and hardware experiments demonstrate that \textsc{BubbleCBF} enables safe autonomous navigation in unknown environments even as nominal trajectory tracking degrades. Future work will consider dynamic environments, online adaptation of the bubble cover as the SDF evolves, and extensions to higher-dimensional configuration spaces.

% \begin{comment}
% \section{To Do's}
% Two things need to be established and I am working on them:

% Pending theoretical results:
% \begin{itemize}
%   \item $\Cc^s_i \cap \Cc_{i+m}^s$ is not empty.
%   \item CLF-CBF compatibility.
%   \item Proof that the safety condition is satisfied for a unicycle with zero velocity.
% \end{itemize}

% Practicals:
% \begin{itemize}
%     \item Check the literature for similar projects, for example: 
%     \begin{enumerate}
%         \item Ames' reduced-order models and CBFs.
%     \end{enumerate}
%     \item Finalize code for unicycle and implement it.
% \end{itemize}
% \end{comment}

%% file: appendix.tex
\section{Appendix}
% \subsection{\bf Probabilistic Bubbles for the Choice of Barrier Functions}

% In practice, robots estimate the SDF from onboard sensors. These estimates closely approximate the true SDF but are inherently probabilistic. We choose our barrier function \(h\) given a probabilistic confidence interval to provide probabilistic guarantees of safety. Any other formulation of \(h\) can be used. 

% \BL{Put this in the experiments. Throughout paper use $h$, this is our choice of $h$.}
% \begin{prop}[Probabilistic safe bubble from estimated SDF]\label{Prop:probbubble}
% Let $S(x)$ be the random signed distance at $x\in\mathbb{R}^n$ with mean $\mu(x)$ and variance $\sigma^2(x)$.
% A subtractive slack of the value
% \[
% \tau(x)=\sqrt{\frac{p}{1-p}}\;\sigma(x)
% \]
% is the minimal amount to reduce to ensure probabilistic safety at confidence level $p\in(0,1)$ for bubbles in (\ref{eq:positional_safe_ball}).
% \
% \end{prop}
% \begin{proof}
% Consider the probabilistic safety bound
% \[\Pr\!\big(S(x)\ge \mu(x)-\tau(x)\big) \ge p \]

% By Cantelli's lower tail bound

% \[\Pr\!\big(S(x)-\mu(x)\leq-\tau(x)\big) \leq \frac{\sigma^2}{\sigma^2 + \tau(x)^2} \]

% By taking the compliment and setting \(p = \frac{\tau(x)^2}{\sigma^2 + \tau(x)^2}\), we minimize \(\tau(x)\) such that \(\tau(x)^* = \sqrt{\frac{p}{1-p}}\sigma(x)\)
% \end{proof}

\subsection{Proof of Theorem~\ref{thm:safe}}\label{app:proof_safe}
 \begin{proof}
  Let $\bfx \in \partial \calS$, i.e., $s(\bfx) = 0$, and define  $\Ic(\bfx) \coloneqq \{i \;|\; s_i(\bfx) = s(\bfx)\}$. For every $i \in \mathcal{I}(\bfx)$, we have $s_i(\bfx)=0$. Define the active barriers \(\mathcal Y_i(\bfx) \coloneqq \{\zeta\in\{b_i,h_i,c\}\mid \zeta(\bfx)=s_i(\bfx)=0\}.\) We show that the controller in~\eqref{eq:thm_feasible} satisfies
$\dot{\zeta}(\bfx,\bfu)>0$ for every $i\in\mathcal I(\bfx)$ and every
$\zeta\in\mathcal Y_i(\bfx)$. 
  %By~\cite[Thm. 4.12]{MA-NA-JC:24-auto}, this is sufficient to prove the theorem. 
Fix $i\in\Ic(\bfx)$ and $\zeta\in\mathcal Y_i(\bfx)$. Then
$\zeta=b_i$, $\zeta=h_i$, or $\zeta=c$, with multiple
barriers possibly being active at the same time.
  % We will construct a choice of $\bfu$ that gives $\dot \zeta(\bfx,\bfu)>0$ for all $i \in  \Ic(\bfx)$ and all possibilities of $\zeta$.

Case (a): $\zeta=b_i$. Since $i\in\Ic(\bfx)$ and
$s_i(\bfx)=\min\{b_i(\bfq),h_i(\bfx),c(\bfv)\}=0$, we have
$h_i(\bfx)\geq0$. Moreover, $\zeta=b_i\in\mathcal Y_i(\bfx)$ implies
$b_i(\bfq)=0$. Combining $h_i(\bfx) \geq 0$ with $b_i(\bfq) = 0$ yields $\dot b_i(\bfq) \geq \epsilon_i > 0$, regardless of $\bfu$.
% Therefore,
% \begin{equation}
% h_i(\bfx)=\dot b_i(\bfx)+\gamma b_i(\bfq)-\epsilon_i=\dot b_i(\bfx)-\epsilon_i\geq0,
% \end{equation}
% and hence
% \begin{equation}
% \dot\zeta(\bfx,\bfu)=\dot b_i(\bfx)=-a2(\bfq-\bfc_i)^\top\bfv\geq\epsilon_i>0,
% \end{equation}
% regardless of $\bfu$.

Case (b): $\zeta = h_i$. So, $\dot \zeta(\bfx,\bfu) = \dot h_i(\bfx,\bfu)$. 
% $= (-2\bfv + -\gamma (\bfq - c_{i}))^\top \bfv -2(\bfq - c_{i})^\top(\mathbf{f}(\bfx)+G(\bfx)u)$. 
By definition of $\calS$, $\bfq \in \bigcap_{i \in \Ic(\bfx)} \calB_i$. Hence, $\Ic(\bfx) \in S_\cap$, and Assumption~\ref{as:interior} ensures $\bfy_{\Ic(\bfx)}$ exists. For brevity, we write $\bfy_{\Ic(\bfx)}$ as $\bfy_\Ic$. 
% We are now ready to define our choice of $\bfu$ as follows.
% \begin{align}
% \bfu \coloneqq a G^{+}(\bfx)(\bfv + \frac{1}{2}\gamma (\bfq - y_{\Ic}))
% \end{align} 
%   with $a < 0$. 
Substituting $\bfu$ in $\dot h_{i}(\bfx)$ gives
\begin{align}
    \dot h_i(\bfx) &= (-2\bfv -2\gamma (\bfq - \bfc_{i}))^\top \bfv -2(\bfq - \bfc_{i})^\top \mathbf{f}(\bfx) \nonumber \\
    % &  \nonumber \\
    & +2k(\bfq - \bfc_{i})^\top (\bfv + \frac{1}{2}\gamma \bfq - \frac{1}{2}\gamma \bfy_{\Ic}) \label{eq:bdotCase2}
\end{align}
We will show that the last term has a positive lower bound, and can thus steer $\dot{h}_i$ positive by setting a sufficiently large positive $k$. From $h_i(\bfx) = 0$, we have:
\begin{equation}\label{eq:sign}
  \begin{aligned}
    % h_i(\bfx) &= -2(\bfq - \bfc_{i})^\top \bfv + \gamma  \left (r_{i}^2-\|\bfq - \bfc_{i}\|^2 \right ) - \epsilon_{i} = 0\notag\\
    % H_i(\bfx) &=-2(\bfq - \bfc_{i})^\top \bfv + \gamma  B_{i}(\bfx) - \epsilon_{i} = 0\notag\\
    % \implies & -2(\bfq - \bfc_{i})^\top \bfv + \gamma  \left (r_{i}^2-\|\bfq - c_{i}\|^2 \right ) - \epsilon_{i} = 0\notag\\
    % \implies & -2(\bfq - \bfc_{i})^\top (\bfv + \frac{1}{2}\gamma \bfq) =\notag\\
    % &\hskip 1cm -\gamma r_{i}^2 - \gamma (\bfq - \bfc_{i})^\top \bfc_{i} + \epsilon_{i} \notag\\
    % \implies & 
    -2&(\bfq - \bfc_{i})^\top (\bfv + \frac{1}{2}\gamma \bfq - \frac{1}{2}\gamma  \bfy_{\Ic}) \\
    &= 
     -\gamma r_{i}^2 - \gamma (\bfq - \bfc_{i})^\top (\bfc_{i} - \bfy_{\Ic}) + \epsilon_i
  \end{aligned}
\end{equation}
  Applying Young's inequality to the second-to-last term yields:
  \begin{align*}
    % -(\bfq - \bfc_{i})^\top (\bfc_{i} - \bfy_{\Ic}) \\
    -(\bfq - \bfc_{i})^\top (\bfc_{i} - \bfy_{\Ic}) &\leq \frac{1}{2}\|\bfq - \bfc_{i}\|^2 + \frac{1}{2} \|\bfc_{i} - \bfy_{\Ic}\|^2\\
    &\leq \frac{1}{2}r_{i}^2 + \frac{1}{2} \delta_{i}^2,
  \end{align*}
  where the last line uses definitions of $\calB$ and $\delta_i$. So by~\eqref{eq:sign},
  \begin{equation}
-2(\bfq - \bfc_{i})^\top (\bfv + \frac{1}{2}\gamma \bfq - \frac{1}{2}\gamma  \bfy_{\Ic}) \leq -\frac{1}{2}\gamma (r_{i}^2 - \delta_{i}^2) + \epsilon_{i} < 0,
  \end{equation}
 % \begin{align*}
 %  &-2(\bfq - \bfc_{i})^\top (\bfv + \frac{1}{2}\gamma \bfq - \frac{1}{2}\gamma  \bfy_{\Ic}) \\
 %    % & \hskip 1cm \leq -\gamma r_{i}^2 + \frac{1}{2}\gamma  r_{i}^2 + \frac{1}{2} \gamma \delta_{i}^2 + \epsilon_{i}\\
 %    & \hskip 1cm \leq -\frac{1}{2}\gamma (r_{i}^2 - \delta_{i}^2) + \epsilon_{i} < 0,
 % \end{align*}
where the last inequality is the premise of the theorem. This gives a positive lower bound on the last term of~\eqref{eq:bdotCase2}, allowing a sufficiently large $k$ to steer $\dot{h}_i$ positive. 
 To see how large $k$ should be, we bound the first two terms of $\dot h_i$ as:
\begin{align*}
  &(-2\bfv  -2\gamma (\bfq - \bfc_{i}))^\top \bfv -2(\bfq - \bfc_{i})^\top \mathbf{f}(\bfx)\\
  \geq & -2\|\bfv\|^2 - 2\gamma \|\bfq - \bfc_{i}\|\|\bfv\| - 2\|\bfq - \bfc_{i}\|\bigl(\|\mathbf{f}_q(\bfq)\|+\|\mathbf{f}_v(\bfv)\|\bigr) \\
  \geq & -2v_{\max}^2 - (2\gamma + 2 k_v ) r_{i} v_{\max} - 2r_{i} k_1,
\end{align*}
where the last inequality is implied by the facts that $\|\bfv\| \leq v_{\max}$ (from $\bfv \in \calC$) and $\|\mathbf{f}(\bfx)\| \leq \|\mathbf{f}_q(\bfq)\| + \|\mathbf{f}_v(\bfv)\| \leq k_1 + k_v\|\bfv\|$. Thus, it is sufficient for $\dot h_i(\bfx) > 0$ that
\begin{equation}
\label{eq:aCond1}
k  > a_1 :=
\max_i
\frac{
2v_{\max}^2
+ 2(\gamma+k_v)r_i v_{\max}
+ 2r_i k_1
}{
\frac{1}{2}\gamma\left(r_i^2-\delta_i^2\right)-\epsilon_i
}
>0.
\end{equation}
Case (c): $\zeta = c$. Thus, $c(\bfv) = v_{\max}^2 - \bfv^\top \bfv = 0$ and $\dot c(\bfv) = -2\bfv^\top \dot {\bfv} = -2\bfv^\top (\mathbf{f}(\bfx) + G(\bfx)\bfu)$. Using $\bfv^\top \bfv = v_{\max}^2$ from $c(\bfv) = 0$, substituting the controller~\eqref{eq:thm_feasible} in $\dot c(\bfv)$ yields
\begin{align*}
  \dot c(\bfv) = -2\bfv^\top \mathbf{f}_q(\bfq) -2\bfv^\top \mathbf{f}_v(\bfv) + 2k \left(v_{\max}^2 + \tfrac{1}{2}\gamma \bfv^\top (\bfq - \bfy_\Ic) \right).
\end{align*}
Using $\|\bfq - \bfy_\Ic\|\leq 2r_i$ for any $i \in \Ic$, we have: 
% where $\Ic$ in $\bfy_\Ic$ is determined from Possibility (b) as long as there is some $i$ with $H_i(\bfx) = 0$. If not, then we choose $\bfu = aG^+(\bfx)\bfv$, which does not affect Possibility (b) since the theorem is proved for it by default. 
% In both cases, for a negative $a$ and keeping in mind that 
\begin{align*}
  \dot c(\bfv) \geq -2v_{\max} k_1 - 2k_v v_{\max}^2 + 2k (v_{\max}^2 - \gamma v_{\max} r_i).
\end{align*}
By the premise $\gamma < \frac{v_{\max}}{r_{\text{max}}}$, it is sufficient for $\dot c(\bfv) > 0$ that
\begin{equation}
k > a_2 :=
\frac{
v_{\max}k_1+k_vv_{\max}^2
}{
v_{\max}^2-\gamma v_{\max}r_{\max}
}
>0.
\label{eq:aCond2}
\end{equation}
Thus, choosing $\bfu$ as in~\eqref{eq:thm_feasible} with $k>\max\{a_1,a_2\}$ satisfies both \eqref{eq:aCond1} and \eqref{eq:aCond2}.
\end{proof}

% \subsection{Proof of Theorem~\ref{thm:inputConstraints}}\label{app:proof_bound}
% \begin{proof}
%     To prove this result we need only show that $\bfu$, defined in the proof of Theorem~\ref{thm:safe}, is in $\Uc$. Recall that $\bfu = a G^+(\bfx)(\bfv + \frac{1}{2}\gamma (\bfq - \bfy_I))$. For the condition $\bfu \in \Uc$ it is sufficient that
%     \begin{align*}
%         |a|k_G(k_vv_{\max} + \frac{1}{2}\gamma \|\bfq - \bfy_{\Ic}\|) \leq u_{\text{max}}       
%     \end{align*}
% But $\|\bfq - \bfy_{\Ic}\| \leq 2r_i$ for any $i \in I$, $|a| > \max\{|a_1|,|a_2|\}$, where $a_1$ and $a_2$ are as per~\eqref{eq:aCond1},~\eqref{eq:aCond2}. 
% \MA{are we setting \(v_{\max}\) to 0 to make it small enough for the input constraint? \(v_{\max}\) has the condition of \(v_{\max}<\gamma r_{max}\) so I want to check if that is the case or not. }
%     Note also that as our choice of $v_{\max}$ decreases,
%     \begin{align*}
%         \max\{|a_1|,|a_2|\} \to \frac{2k_1\gamma r_i}{\frac{1}{2}\gamma (r_i^2 - \delta_i^2) - \epsilon_i^2}
%     \end{align*}
%     \MA{should this be $\epsilon^2$ or $\epsilon$? before we defined $a_1$ to have $\frac{1}{2}\gamma(r_i^2 - \delta_i^2)-\epsilon_i$ in the denominator}
%     with the $i$ relevant to the position $\bfx$ at which $\bfu$ is calculated. So if
%     \begin{align*}
%         \max\{|a_1|,|a_2|\}k_G\gamma r_i  =  \frac{2k_1k_G\gamma r_i^2}{\frac{1}{2}\gamma (r_i^2 - \delta_i^2) - \epsilon_i^2}< u_{\text{max}} 
%  \end{align*}
%  then there is a low enough choice of $v_{\max}$ and of $a$ that is sufficient for $\bfu$ to be in $\Uc$.
% \end{proof}
\subsection{Proof of Theorem~\ref{thm:inputConstraints}}
\label{app:proof_bound}
\begin{proof}
It suffices to show that the controller in~\eqref{eq:thm_feasible} belongs to $\Uc$. Since
$\|G^\dagger(\bfx)\|\leq k_G$, $\|\bfv\|\leq v_{\max}$, and
$\|\bfq-\bfy_{\Ic}\|\leq 2r_{\max}$,
\begin{equation}
\|\bfu\|
\leq
k k_G(v_{\max}+\gamma r_{\max}).
\end{equation}
From~\eqref{eq:aCond1}--\eqref{eq:aCond2}, it suffices to choose
\begin{equation}
\max\{a_1,a_2\}<k
\leq
\frac{u_{\max}}
{k_G(v_{\max}+\gamma r_{\max})}.
\end{equation}
Such a $k>0$ exists provided that
\begin{equation}
u_{\max}
>
k_G(v_{\max}+\gamma r_{\max})\max\{a_1,a_2\}.
\end{equation}
Hence, $\| \bfu \| < u_{\max}$, and the QP in \eqref{eq:qp_cbf} remains feasible.
\end{proof}

% \subsection{Proof of Corollary~\ref{thm:diffdrive_bounds}}
% \label{app:diffdrive_bounds}
% \begin{proof}
% For the differential-drive dynamics~\eqref{eq:diffdrive},
% $\bff(\bfx)=0$, and hence $k_1=k_v=0$. Fix $c>1$ and
% $\rho_i\in(0,1)$, and choose
% \[
% v_{\max}=c\gamma r_{\max},
% \qquad
% \epsilon_i
% =
% \rho_i\frac{\gamma}{2}(r_i^2-\delta_i^2).
% \]
% Then the conditions of Theorem~\ref{thm:safe} are satisfied. Moreover,
% the gain bound\BL{add eqref to which bound you are talking about} in the proof of Theorem~\ref{thm:inputConstraints}
% reduces to\BL{I can see the lower bound on $u_{\max}$ should be $k_G (1 + c) \gamma r_{\max} \max{ \{a_1, a_2\} }$. Where did $k_G$ go in $L$?}
% \[
% \max\left\{
% a_1,a_2
% \right\} \gamma L,
% \]
% where
% \[
% L:=
% \max_i
% \frac{
% 2(2c^2r_{\max}^2+cr_ir_{\max})
% }{
% (1-\rho_i)(r_i^2-\delta_i^2)
% }
% \]
% is independent of $\gamma$. Therefore, the sufficient input condition is
% \[
% u_{\max}
% >
% k_GL(c+1)r_{\max}\gamma^2.
% \]
% Since the right-hand side converges to zero as $\gamma\to0$, for every
% $u_{\max}>0$ we may choose
% \[
% \gamma
% <
% \sqrt{
% \frac{u_{\max}}
% {k_GL(c+1)r_{\max}}
% }.
% \]
% Hence, the constrained quadratic program can be made feasible for any
% positive input limit.
% \end{proof}
\subsection{Proof of Corollary~\ref{thm:diffdrive_bounds}}
\label{app:diffdrive_bounds}
\begin{proof}
For the differential-drive dynamics~\eqref{eq:diffdrive2},
$\bff(\bfx)=0$, so $k_1=k_v=0$ and, from~\eqref{eq:aCond2}, $a_2=0$. Fix $\eta>1$ and $\rho_i\in(0,1)$, and choose
\begin{equation}
v_{\max}=\eta\gamma r_{\max},
\qquad
\epsilon_i
=
\rho_i\frac{\gamma}{2}(r_i^2-\delta_i^2).
\end{equation}
These choices satisfy the conditions of Theorem~\ref{thm:safe}. Substituting
them into~\eqref{eq:aCond1} gives
\begin{equation}
a_1
=
\gamma
\max_i
\frac{
4(\eta^2r_{\max}^2+\eta r_ir_{\max})
}{
(1-\rho_i)(r_i^2-\delta_i^2)
}
=:\gamma L,
\end{equation}
where $L$ is independent of $\gamma$. Hence,
$\max\{a_1,a_2\}=\gamma L$.

The sufficient condition of Theorem~\ref{thm:inputConstraints} becomes
\begin{equation}
\label{eq:29}
u_{\max}
>
k_GL(\eta+1)r_{\max}\gamma^2.
\end{equation}
On states with $\|\bfv\|\ge v_{\min}$,
$k_G=\max\{1,(\beta\gamma r_{\max})^{-1}\}$. Thus,~\eqref{eq:29} holds for any $u_{\max}>0$ by choosing
\begin{equation}
0<\gamma<
\min\left\{
\sqrt{\frac{u_{\max}}{L(\eta+1)r_{\max}}},
\frac{\beta u_{\max}}{L(\eta+1)}
\right\}.
\end{equation}
Therefore, the QP is feasible for any positive input limit
by choosing $\gamma$ sufficiently small.
\end{proof}

% \NA{This is written confusingly. What are we showing? That the controller in \eqref{eq:thm_feasible} ensures $\dot \zeta(\bfx, \bfu) > 0$ for any function of $b_i, h_i,$ or $c$ that equals $0 = s_i(\bfx)$ at $\bfx$? Is $\zeta$ only $b_i$, $h_i$, or $c$ or it can be a more complex function of these?}